\documentclass[twoside]{article}

\usepackage[accepted]{aistats2021}
\usepackage[square,numbers]{natbib}

\usepackage[utf8]{inputenc} 
\usepackage[T1]{fontenc}    
\usepackage{url}            
\usepackage{booktabs}       
\usepackage{amsfonts}       
\usepackage{graphicx}
\usepackage{nicefrac}       
\usepackage{microtype}      
\usepackage{amsmath}
\usepackage{amsthm}
\usepackage{mathrsfs}
\usepackage{bm,bbm}
\usepackage{mathtools}
\usepackage{subcaption}
\usepackage{algpseudocode,algorithm,algorithmicx}
\usepackage{array, adjustbox}
\usepackage{caption}
\usepackage{hyperref}       
\usepackage{cleveref}
\usepackage{mathrsfs}
\hypersetup{
    colorlinks,
    linkcolor={blue},
    citecolor={blue},
    urlcolor={blue}
}

\newcommand{\wass}{\mathcal{W}}

\newcommand{\hmun}{\hat{\mu}_n}
\newcommand{\hnun}{\hat{\nu}_m}
\newcommand{\bmU}{\mathcal{U}}

\newcommand{\bmA}{\mathcal{A}}

\newcommand{\bmO}{\mathcal{O}}
\newcommand{\bmD}{\mathcal{D}}

\newcommand{\Wmom}{\mathcal{W}_{\text{MoM}}}

\newcommand{\Wmoudiag}{\mathcal{W}_{\text{MoU-diag}}}
\newcommand{\Wmou}{\Wmoudiag}
\newcommand{\Wmoucb}{\mathcal{W}_{\text{MoU}}}

\newcommand{\tWmom}{\widetilde{\mathcal{W}}_{\text{MoM}}}
\newcommand{\tWmou}{\widetilde{\mathcal{W}}_{\text{MoU-diag}}}
\newcommand{\tWmoucb}{\widetilde{\mathcal{W}}_{\text{MoU}}}
\newcommand{\tWmoudiag}{\widetilde{\mathcal{W}}_{\text{MoU-diag}}}

\newcommand{\tW}{\widetilde{\mathcal{W}}}

\newcommand{\tauX}{\tau_{\mathbf{X}}}
\newcommand{\tauY}{\tau_{\mathbf{Y}}}
\newcommand{\KX}{K_{\mathbf{X}}}
\newcommand{\KY}{K_{\mathbf{Y}}}
\newcommand{\BX}{B_{\mathbf{X}}}
\newcommand{\BY}{B_{\mathbf{Y}}}

\newcommand{\X}{\mathbf{X}}

\newcommand{\Y}{\mathbf{Y}}

\newcommand{\compact}{\mathcal{K}}

\algrenewcommand\algorithmicrequire{\textbf{Precondition:}}
\algrenewcommand\algorithmicensure{\textbf{Postcondition:}}
\newtheorem{theorem}{Theorem}
\newtheorem{definition}[theorem]{Definition}
\newtheorem{proposition}[theorem]{Proposition}
\newtheorem{lemme}[theorem]{Lemma}
\newtheorem{remark}[theorem]{Remark}
\newtheorem{assumption}[theorem]{Assumption}

\begin{document}
%

%
\runningauthor{Guillaume Staerman, Pierre Laforgue, Pavlo Mozharovskyi, Florence d'Alch\'{e}-Buc}

\twocolumn[

\aistatstitle{When OT meets MoM: Robust estimation of Wasserstein Distance}

\aistatsauthor{Guillaume Staerman \And Pierre Laforgue \And Pavlo Mozharovskyi \And Florence d'Alch\'{e}-Buc }

\aistatsaddress{ LTCI, T\'{e}l\'{e}com Paris, Institut Polytechnique de Paris } ]

\begin{abstract}
Issued from Optimal Transport, the Wasserstein distance has gained importance in Machine Learning due to its appealing geometrical properties and the increasing availability of efficient approximations.
It owes its recent ubiquity in generative modelling and variational inference to its ability to cope with distributions having non overlapping support.
In this work, we consider the problem of estimating the Wasserstein distance between two probability distributions when observations are polluted by outliers. To that end, we investigate how to leverage a Medians of Means (MoM) approach to provide robust estimates. 
Exploiting the dual Kantorovitch formulation of the Wasserstein distance, we introduce and discuss novel MoM-based robust estimators whose consistency is studied under a data contamination model and for which convergence rates are provided. Beyond computational issues, the choice of the partition size, \textit{i.e.,} the unique parameter of theses robust estimators, is investigated in numerical experiments. Furthermore, these MoM estimators make Wasserstein Generative Adversarial Network (WGAN) robust to outliers, as witnessed by an empirical study on two benchmarks CIFAR10 and Fashion MNIST.

\end{abstract}

\section{Introduction}
Computing distances between probability distributions has become a central question in numerous modern Machine Learning applications, ranging from generative modeling to clustering.
%
%
%
Optimal Transport (OT) \citep{Villani,Santambrogio} offers an appealing and insightful tool to solve this problem, building upon the Wasserstein distance.
%
%
%
Given two probability distributions, the latter is defined in terms of the solution to the Monge-Kantorovich optimal mass transportation problem.
%
%
%
Interestingly, it relies on a ground distance between points to build a distance between probability distributions \citep{Peyre}. For that reason, the Wasserstein distance stands out from the divergences usually exploited in generative modeling, like the f-divergences \citep{csiszar,nguyen2009}, by its ability to take into account the underlying geometry of the space, capturing the difference between probability distributions even when they have non-overlapping supports. This appealing property has been successfully exploited in Generative Adversarial Networks (GANs) \citep{goodfellow,arjovsky2017,gulrajani2017improved}, as well as in Variational Autoencoders (VAEs) \citep{bousquetetal17}, where the Wasserstein distance can advantageously replace an f-divergence as the loss function. 
Many other applications \citep{courty2017,flamary18,genevay18} rely on the entropic-regularized approximations introduced by \cite{cuturi13}, which has considerably alleviated the inherent computational complexity of the Wasserstein distance in the discrete case, by drawing on the Sinkhorn-Knopp algorithm. A common feature to almost all these works is that the Wasserstein distance is estimated from finite samples. While this problem has long been theoretically studied under the i.i.d. assumption \citep{dudley1969,bassetti2006,weed2019}, it has never been tackled through the lens of robustness to outliers, a crucial issue in Reliable Machine Learning. Indeed, data is nowadays collected at a large scale in unmastered acquisition conditions, and through a large variety of devices and platforms. The resulting datasets often present undesirable influential observations, whether they are errors or rare observations. The presence of corrupted data may heavily damage the quality of estimators, calling for dedicated methods such as JS/TV-GANs \citep{gao2018robust} in the particular case of robust shift-parameter estimation, Robust Divergences in variational inference \citep{futami}, or more general tools from robust statistics \citep{Huber2009}.

The aim of this work is to propose outliers-robust estimators of the Wasserstein distance, and illustrate their application in generative modeling. To that end, we explore how to combine a Median-of-Means approach with Optimal Transport. The Median-of-Means (MoM) is a robust mean estimator firstly introduced in complexity theory during the 1980s \cite{nemirovsky1983problem,jerrum1986random,alon1999space}.
Following the seminal deviation study by \cite{catoni2012challenging}, MoM has lately witnessed a surge of interest, mainly due to its attractive sub-gaussian behavior, under the sole assumption that the underlying distribution has finite variance \citep{devroye2016sub}. Originally devoted to scalar random variables, MoM has notably been extended to random vectors \citep{minsker2015geometric,hsu2016loss,lugosi2017sub} and $U$-statistics \citep{joly2016robust,laforgue2019medians}. As a natural alternative to the empirical mean, MoM has become the cornerstone of several robust learning procedures in heavy-tailed situations, including bandits \citep{bubeck2013bandits} and MoM-tournaments \citep{lugosi2019risk}. A more recent line of work now focuses on MoM's ability to deal with outliers. Aside from concentration results in a contaminated context \cite{depersin2019robust,papier2}, it has yielded promising applications in robust mean embedding \citep{monk}, and the more general MoM-minimization framework \citep{lecue2018robust}.

In this paper, we introduce and study outliers-robust estimators of the Wasserstein distance based on the MoM methodology. Our contribution is threefold:
\begin{itemize}
\item Focusing on the Kantorovich-Rubinstein duality \citep{kantorovich1958}, we present three novel MoM-based estimators, leveraging in particular Medians of $U$-statistics (MoU). In the realistic setting of contaminated data, we show their strong consistency, and provide non-asymptotic bounds as well.
\item We propose a dedicated algorithm to compute these three estimators in practice. Applied on a parametric family of Lipschitz functions, \textit{e.g.} neural networks with clipped weights, it performs a MoM/MoU gradient descent algorithm.  A sensitivity analysis of the unique parameter of these estimators is also provided throught numerical experiments on toy datasets.
\item  We robustify WGANs (w.r.t. outliers) using a MoM-based estimator as loss function. We show the benefits of this approach through convincing numerical results on two contaminated well known benchmarks: CIFAR10 and Fashion MNIST.
\end{itemize}


\section{Background and preliminaries} \label{Back}
Before introducing the problem to be addressed, we recall some key notions about the Wasserstein distance and the Medians-of-Means estimator. Let $\mathcal{X}$ and $\mathcal{Y}$ be subsets of $\mathbb{R}^d$, for some $d \in \mathbb{N}^*$. We denote by $\mathcal{M}^1_+(\mathcal{X})$ the space of all probability measures on $\mathcal{X}$, and consider two distributions $\mu$ and $\nu$ from $\mathcal{M}^1_+(\mathcal{X})$ and $\mathcal{M}^1_+(\mathcal{Y})$. For any $K \in \mathbb{N}^*$, the median of $\{z_1, \ldots, z_K\} \in \mathbb{R}^K$ is denoted by $\underset{1 \le k \le K}{\text{med}}\{z_k\}$.

\subsection{Wasserstein Distance}

Given $p  \in [1, \infty)$, the Wasserstein distance of order $p$ between two arbitrary measures $\mu$ and $\nu$ is defined through the resolution of the Monge-Kantorovitch mass transportation problem \citep{Villani,Peyre}:
\begin{equation}\label{OT-primal}
\wass_p(\mu, \nu)= \underset{ \pi~\in~\bmU(\mu,\nu)}{\min} \left( \int_{\mathcal{X} \times \mathcal{Y}}   \| x - y\|^p d\pi(x\times y) \right)^{1/p}\ ,
\end{equation}
where $\bmU(\mu,\nu)= \{ \pi \in \mathcal{M}^1_+(\mathcal{X} \times \mathcal{Y}): \; \; \int \pi(x,y)dy =\mu(x) ;\int \pi(x,y) dx=\nu(y) \}$  is the set of joint probability distributions with marginals $\mu$ and $\nu$. In the remainder of this paper, we focus on the Wasserstein of order 1, $\wass_1$, omitting the subscript $1$ for notation simplicity.
%
By the dual Kantorovich-Rubinstein formulation \cite{kantorovich1958}, with $\mathcal{B}_L$ the unit ball of the Lipschitz functions space, a useful rewriting of the $1$-Wasserstein distance is:
\begin{equation}\label{OT-dual}
\wass(\mu,\nu)=\underset{\phi \in \mathcal{B}_L}{\sup} \; \; \mathbb{E}_\mu \left[\phi(X)\right]-\mathbb{E}_\nu \left[\phi(Y)\right].
\end{equation}


Of particular interest  is the problem of estimating the Wasserstein distance between $\mu$ and $\nu$ given a finite number of observations.
The usual assumption is to rely upon two samples $\mathbf{X}=\{X_1, \ldots, X_n\}$ and $ \mathbf{Y}=\{Y_1, \ldots, Y_m\}$, composed of i.i.d. realizations drawn respectively from $\mu$ and $\nu$.
The corresponding empirical distributions denoted by $\hmun = (1/n) \sum_{i=1}^n \delta_{X_i}$, and $\hnun= (1/m) \sum_{j=1}^m \delta_{Y_j}$.
The natural questions are then: \textit{how to compute the estimator $\wass(\hmun,\hnun)$, and does it converge towards $\wass(\mu,\nu)$?}
In the dual formulation \eqref{OT-dual}, computing $\wass(\hmun, \hnun)$ is equivalent to replace the expectations with empirical means.
The unit ball of Lipschitz functions can be replaced with a parameterized family of Lipschitz functions, more amenable for learning when $\wass(\hmun, \hnun)$ is used as a loss function, see \textit{e.g.} Wasserstein GANs \cite{arjovsky2017}.
%
%
%
%
From the theoretical side, a substantial number of works have studied the convergence of $\wass(\hmun, \hnun)$ under the i.i.d. setting described above. Statistical rates of convergence of the original OT problem are known to be slow rates with respect to the dimension $d$ of the input space, \textit{i.e.} they are of order $O(n^{-1/d})$ \citep{dudley1969,bassetti2006,weed2019,boissard2011,FG15}.

\subsection{Median-of-Means}\label{s:intro_MoM}

Given an i.i.d. sample $\mathbf{X} = \{ X_1, \ldots, X_n\}$ drawn from $\mu$, the Median-of-Means (MoM) is an estimator of $\mathbb{E}_\mu[X]$ built as follows. First, choose $\KX \le n$, and partition $\{1,\ldots, n \}$ into $\KX$ disjoint blocks $\mathcal{B}_1^{\mathbf{X}}, \ldots, \mathcal{B}_{\KX}^{\mathbf{X}}$ of size $\BX = n / \KX$. If $n$ cannot be divided by $\KX$, some observations may be removed. Then, empirical means are computed on each of the $\KX$ blocks. The estimator returned is finally the median of the empirical means thus computed. For a function $\Phi\colon \mathcal{X} \rightarrow \mathbb{R}$, the MoM estimator of $\mathbb{E}_\mu[\Phi(X)]$ is then formally given by:
\begin{equation}\label{eq:MoM}
\text{MoM}_{\mathbf{X}}[\Phi]=\underset{1\leq k\leq \KX}{\text{med}} \Bigl\{ \frac{1}{\BX}\sum_{i \in \mathcal{B}_k^{\mathbf{X}}} \Phi(X_i) \Bigr\}.
\end{equation}
This estimator provides an attractive alternative to the sample mean $\overline{\Phi}_{\mathbf{X}}= (1/n)\sum_{i=1}^n \Phi(X_i)$ for robust learning. Indeed, it has been shown to (i) exhibit a sub-Gaussian behavior under only a finite variance assumption, making it particularly suited to heavy-tailed distributions, and (ii) be non-sensitive to outliers.
MoM also nicely adapts to multisample $U$-statistics of arbitrary degrees \citep{Lee90}. Indeed, assume that one is interested in estimating $\mathbb{E}_{\mu \otimes \nu}[h(X, Y)]$, for some kernel $h \colon \mathcal{X} \times \mathcal{Y} \rightarrow \mathbb{R}$. Given the samples $\mathbf{X} = \{ X_1, \ldots, X_n\}$ and $\mathbf{Y} = \{ Y_1, \ldots, Y_m\}$, a natural idea then consists in partitioning both $\{1,\ldots, n \}$ and $\{1,\ldots, m \}$ into $\mathcal{B}_1^{\mathbf{X}}, \ldots, \mathcal{B}_{\KX}^{\mathbf{X}}$ and $\mathcal{B}_1^{\mathbf{Y}}, \ldots, \mathcal{B}_{\KY}^{\mathbf{Y}}$ respectively, with $\KY \le m$, and $\BY = m / \KY$. One may then compute $U$-statistics on each pair of blocks $(k, l)$ for $k \le \KX$ and $l \le \KY$, and return the median of the $\KX\times \KY$ $U$-statistics. However, this construction introduces dependence between the base estimators, making the theoretical study more difficult. An alternative then consists in choosing $\KX = \KY = K$, and considering only the diagonal blocks (see Figure \ref{fig:mou_diag}). These two estimators are referred to as (diagonal) Median-of-$U$-statistics (MoU), and using $\mathcal{B}_{k,l}^{\mathbf{XY}}$ to denote the block of tuples $(X_i,Y_j)$ such that $X_i \in \mathcal{B}_k^{\mathbf{X}}$ and $Y_j \in \mathcal{B}_l^{\mathbf{Y}}$, they are formally given by:
%
%
\begin{align*}
\text{MoU}_{\mathbf{XY}}[h] &=\underset{\substack{1\leq k\leq \KX\\1\leq l\leq \KY }}{\text{med}} \Bigl\{ \frac{1}{\BX\BY} \underset{(i,j) \in \mathcal{B}_{k,l}^{\mathbf{XY}} }{\sum} h(X_i, Y_j) \Bigr\},\\[0.2cm]
\text{MoU}^\text{diag}_{\mathbf{XY}}[h] &=\underset{1\leq k\leq K}{\text{med}} \Bigl\{ \frac{1}{\BX\BY} \underset{(i,j) \in \mathcal{B}_{k,k}^{\mathbf{XY}} }{\sum} h(X_i, Y_j) \Bigr\}.
\end{align*}

\section{When Wasserstein meets MoM}
\label{MoM}
In this section, we investigate how MoM estimators can be leveraged to define and analyze new estimators of $\wass(\mu,\nu)$ that exhibit strong theoretical guarantees in presence of outliers.
In order to assess robustness, we place ourselves in the realistic $\mathcal{O} \cup \mathcal{I}$ framework, see \textit{e.g.} \cite{Huber2009,lecue2017robust}, devoted to data contamination.
In this setting, the i.i.d. assumption is relaxed, and  the following assumption is instead adopted.

\begin{assumption}\label{ass:contamination}
Sample $\mathbf{X}$ is polluted with $n_{\bmO} < n/2$ (possibly adversarial) outliers. The remaining $n - n_{\bmO}$ points are informative data, or \emph{inliers}, independently distributed according to $\mu$. A similar assumption is made on $\mathbf{Y}$, which is supposed to contain $m_{\bmO} < m/2$ arbitrary outliers, and $m - m_{\bmO}$ inliers drawn from $\nu$. Inliers are assumed to lie in a compact set $\compact \subset \mathbb{R}^d$. In contrast, no assumption is made on the outliers, that may not be bounded. The proportions of outliers in samples $\mathbf{X}$ and $\mathbf{Y}$ are denoted by $\tau_{\mathbf{X}}= n_{\bmO} / n$  and $\tau_{\mathbf{Y}}=m_{\bmO} / m$ respectively.
\end{assumption}



\subsection{MoM and MoU-based estimators}
\label{sec:estimators}

Starting from the expression of the dual expression \eqref{OT-dual}, we observe that it can be considered with a two-fold perspective. The first one consists in considering the Wasserstein distance as the supremum of the difference between two expected values. The second one, obtained by linearity of the expectation, rather regards $\mathcal{W}(\mu, \nu)$ as the supremum of single expected values, but taken with respect to the tuple $(X, Y)$, and associated to the kernel: $h_\phi\colon (X, Y) \mapsto \phi(X) - \phi(Y)$.

Although quite elementary at first sight, this two-fold perspective gains complexity when applied to the empirical distributions $\hat{\mu}_n$ and $\hnun$.
Indeed, following the first perspective, the natural estimator obtained is the supremum of the differences between two empirical averages, while the second one leads to the supremum of $2$-samples $U$-statistics of degrees $(1, 1)$ and kernels $h_{\phi}$.
So far, both points of view are strictly equivalent by linearity of the expectation and the empirical mean.
However, this equivalence breaks down as soon as non-linearities are introduced, through MoM-like estimators for instance.
We therefore introduce three distinct estimators of $\mathcal{W}(\mu, \nu)$, that differ upon which estimator of Section \ref{s:intro_MoM} is used.

\begin{definition}\label{def:estimators}
We define the Median-of-Means and the Median-of-$U$-statistics estimators of the $1$-Wasserstein distance as follows:
\begin{align*}
\mathcal{W}_\mathrm{MoM}(\hat{\mu}_n, \hat{\nu}_m) &= \sup_{\phi \in \mathcal{B}_L} \{ \mathrm{MoM}_{\mathbf{X}}[\phi] - \mathrm{MoM}_{\mathbf{Y}}[\phi] \},\\[0.2cm]
\mathcal{W}_\mathrm{MoU}(\hat{\mu}_n, \hat{\nu}_m) &= \sup_{\phi \in \mathcal{B}_L} \{ \mathrm{MoU}_{\mathbf{XY}}[h_\phi] \},\\[0.2cm]
\mathcal{W}_\mathrm{MoU-diag}(\hat{\mu}_n, \hat{\nu}_m) &= \sup_{\phi \in \mathcal{B}_L} \{ \mathrm{MoU}^\mathrm{diag}_{\mathbf{XY}}[h_\phi] \}.
\end{align*}
\end{definition}

While $\Wmom$ relies on the difference between individual median blocks, $\Wmou$ considers the median over all possible combinations of blocks between $\mathbf{X}$ and $\mathbf{Y}$.
As an intermediate step, $\Wmoudiag$ looks after diagonal blocks only. 
The latter formulation is used in \cite{monk} to derive robust mean embedding and Maximum Mean Discrepancy estimators.
The theoretical analysis is made simpler by the independence between the blocks, but the estimator suffers from an increased variance due to the important loss of information, see Figure \ref{fig:mou_diag} and \cite{joly2016robust}.
It should be noticed however that $\Wmoudiag$ enjoys a much lower computational cost in practice.

%

One elegant way to combine both benefits, \textit{i.e.} small loss of information and low computational cost, is to consider randomized blocks \cite{laforgue2019medians}.
Instead of partitioning the dataset, this method builds blocks by sampling them independently through simple Sampling Without Rejection (SWoR).
One consequence is the possibility for the randomized blocks to overlap (see $\mathscr{B}^{\mathbf{X}}_1$, $\mathscr{B}^{\mathbf{X}}_2$, $\mathscr{B}^{\mathbf{X}}_3$ in Figure \ref{fig:mom_morm}), making the estimator's concentration analysis more difficult.
Nevertheless, guarantees similar to that of MoM have been established (up to constants), and the extension to $2$-sample $U$-statistics built on randomized blocks allows for a better exploration of the grid than through $\mathrm{MoU}^\text{diag}$, see Figure \ref{fig:mou_rand}.
However, despite the possibility to reach every part of the grid, the exploration scheme illustrated in Figure \ref{fig:mou_rand} have a fixed structure (\textit{e.g.} always $3$ cells per column, $4$ cells per row).
The \emph{totally free} alternative, as depicted in Figure \ref{fig:mou_incomp}, consists in sampling directly from the pairs of observations, which generates incomplete $U$-statistics.
If no theoretical guarantees have been established for this extension due to the complex replication setting between blocks, it still benefits from good empirical results \cite{laforgue2019medians}, consistent with the grid covering it allows.

Another important question to be addressed is: \textit{how to handle the non-differentiability introduced by the median operator?}
Indeed, the Wasserstein distance often acts as a loss function, \textit{e.g.} in generative modeling (VAEs, GANs), and optimizing through a MoM/MoU-based criterion then becomes crucial.
One answer is to uses a MoM-gradient descent algorithm \citep{lecue2018robust}.
It consists in performing a mini-batch gradient step based on the median block.
In order to avoid local minima, authors propose shuffle the partition at each step of the descent, leading to the minimization of an expected MoM loss (w.r.t. the shuffling) that is more stable.
Notice that this method goes beyond random partitions, and easily adapts to the randomized extensions discussed above.


\begin{figure*}[!t]
\centering
\begin{subfigure}{.35\textwidth}
  \centering
  \includegraphics[width=0.65\textwidth]{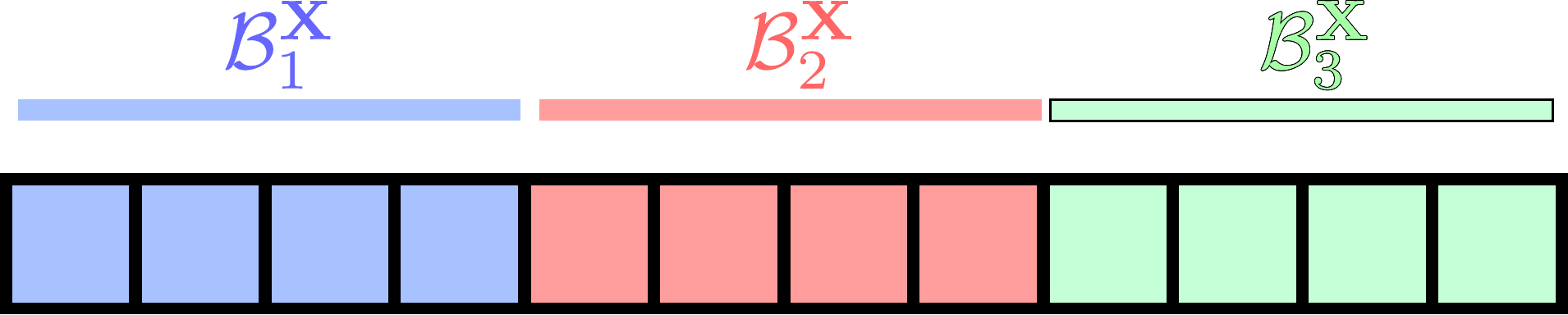}\\[0.3cm]
  \includegraphics[width=0.65\textwidth]{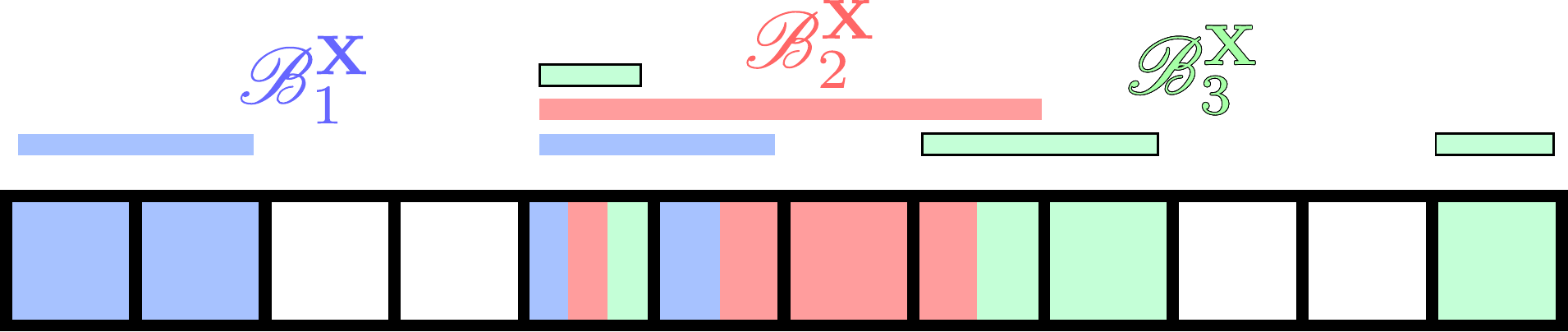}\\[0.4cm]
  \caption{1D standard and randomized MoM}
  \label{fig:mom_morm}
\end{subfigure}
\hfill
\begin{subfigure}{.3\textwidth}
  \centering
  \includegraphics[width=0.65\textwidth]{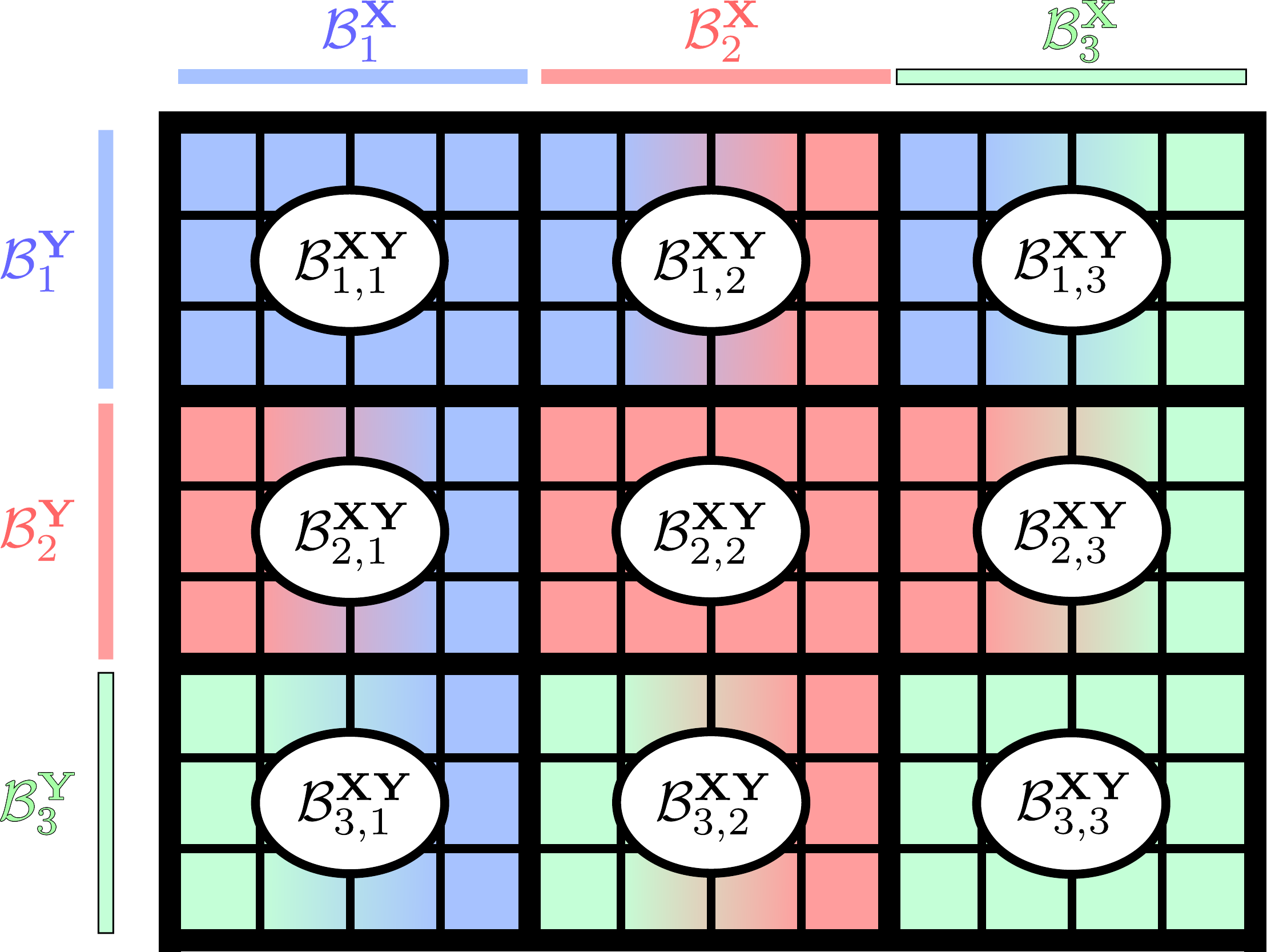}
  \caption{$\mathrm{MoU}_\mathbf{XY}$}
  \label{fig:mou}
\end{subfigure}
\hfill
\begin{subfigure}{.3\textwidth}
  \centering
  \includegraphics[width=0.65\textwidth]{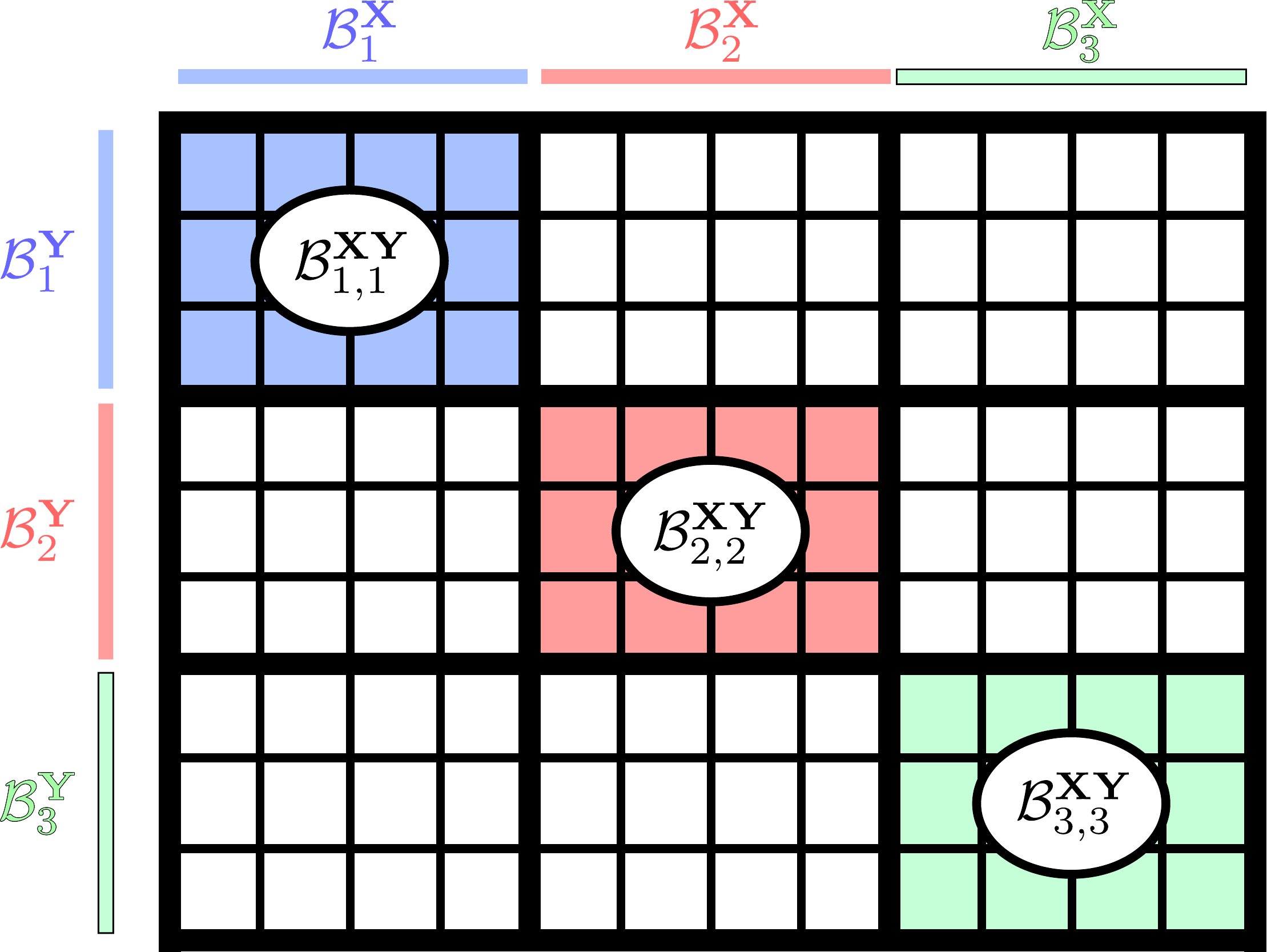}
  \caption{$\mathrm{MoU}^\text{diag}_\mathbf{XY}$}
  \label{fig:mou_diag}
\end{subfigure}\\[0.4cm]
\begin{subfigure}{.35\textwidth}
  \centering
  \vspace{-0.3cm}
  \includegraphics[width=0.55\textwidth]{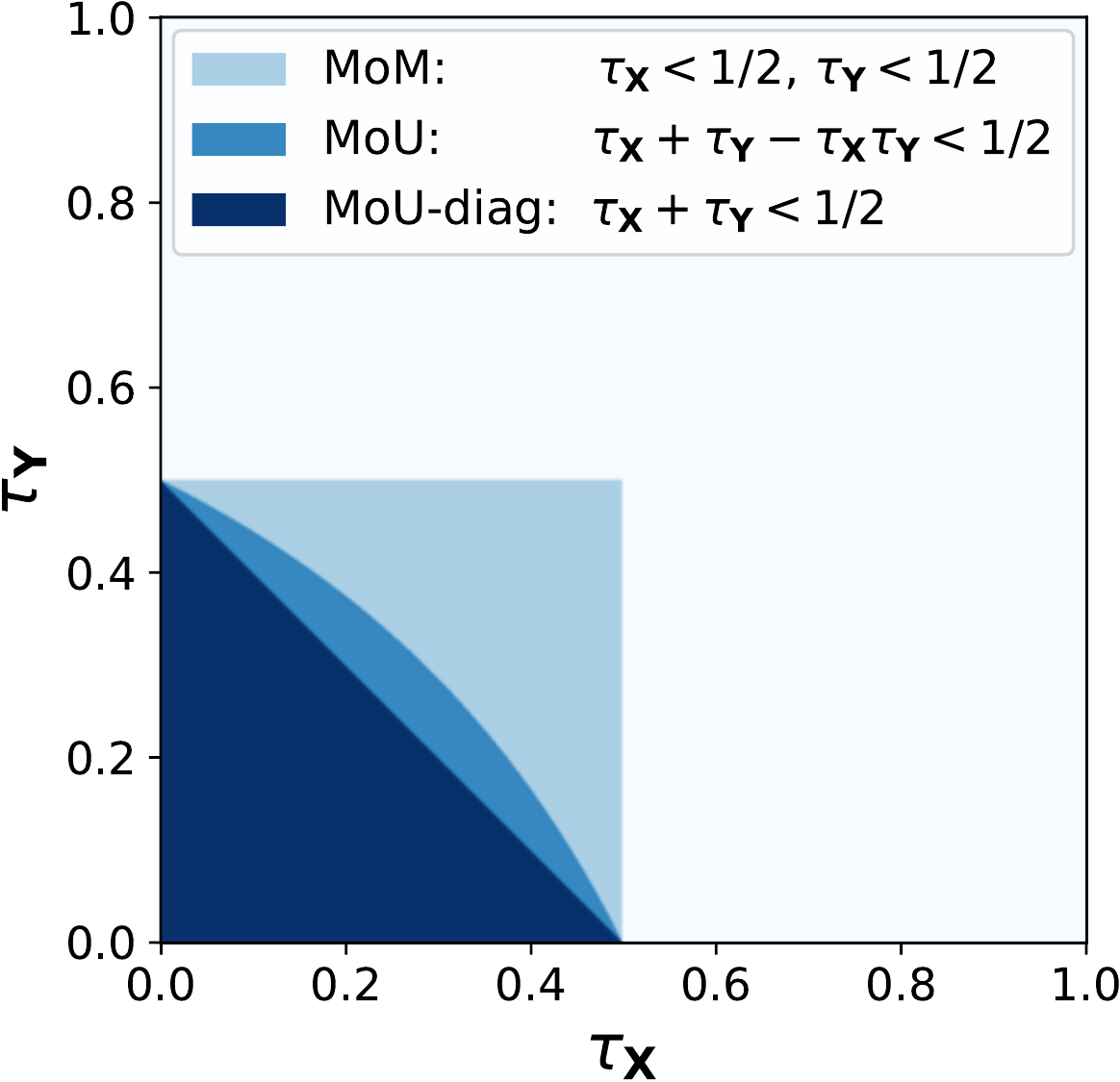}\\[-0.25cm]
  \caption{Admitted proportion of outliers}
  \label{fig:outliers}
\end{subfigure}
\hfill
\begin{subfigure}{.3\textwidth}
  \centering
  \includegraphics[width=0.65\textwidth]{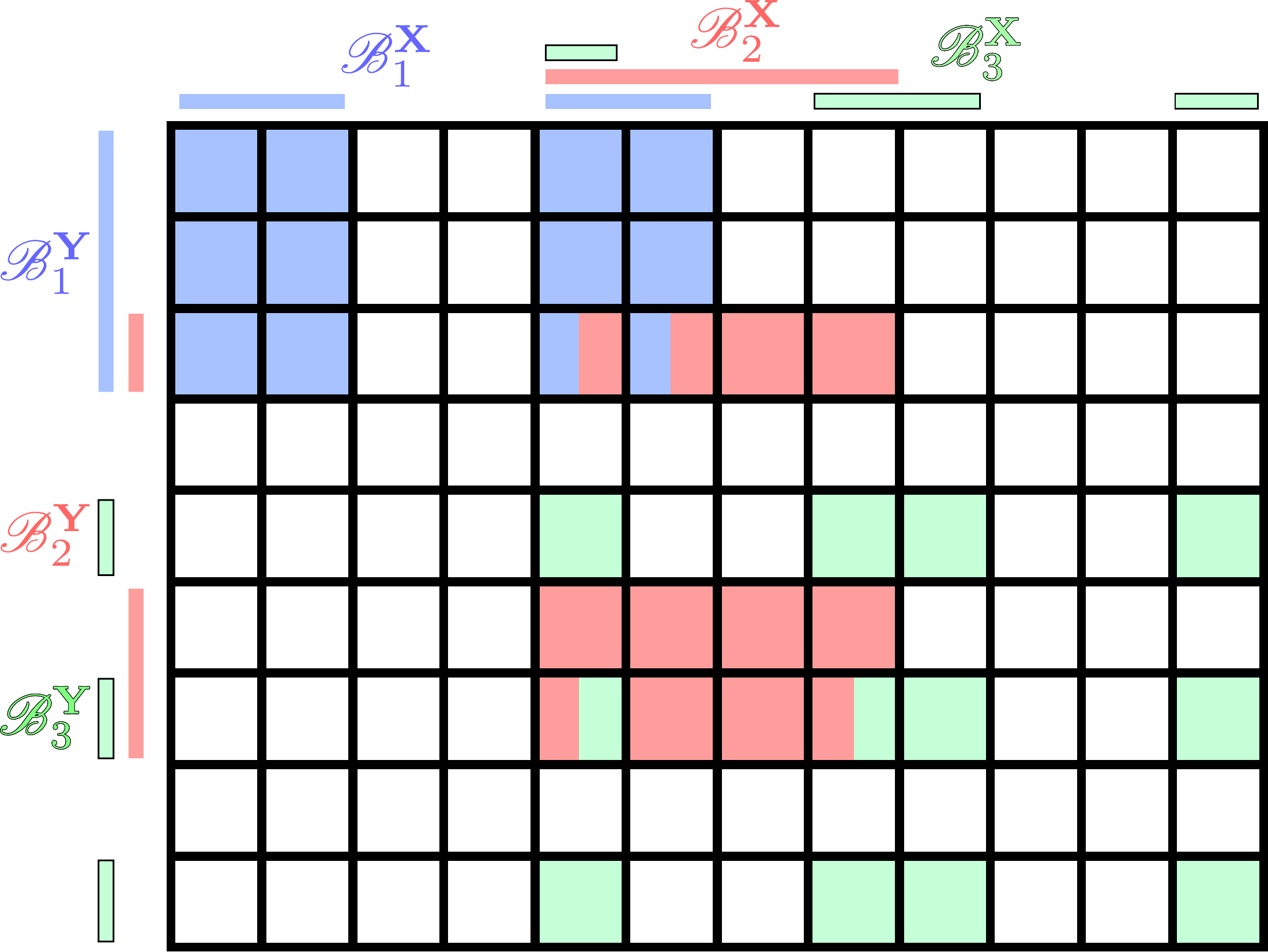}
  \caption{MoU based on randomized blocks}
  \label{fig:mou_rand}
\end{subfigure}
\hfill
\begin{subfigure}{.3\textwidth}
  \centering
  \includegraphics[width=0.65\textwidth]{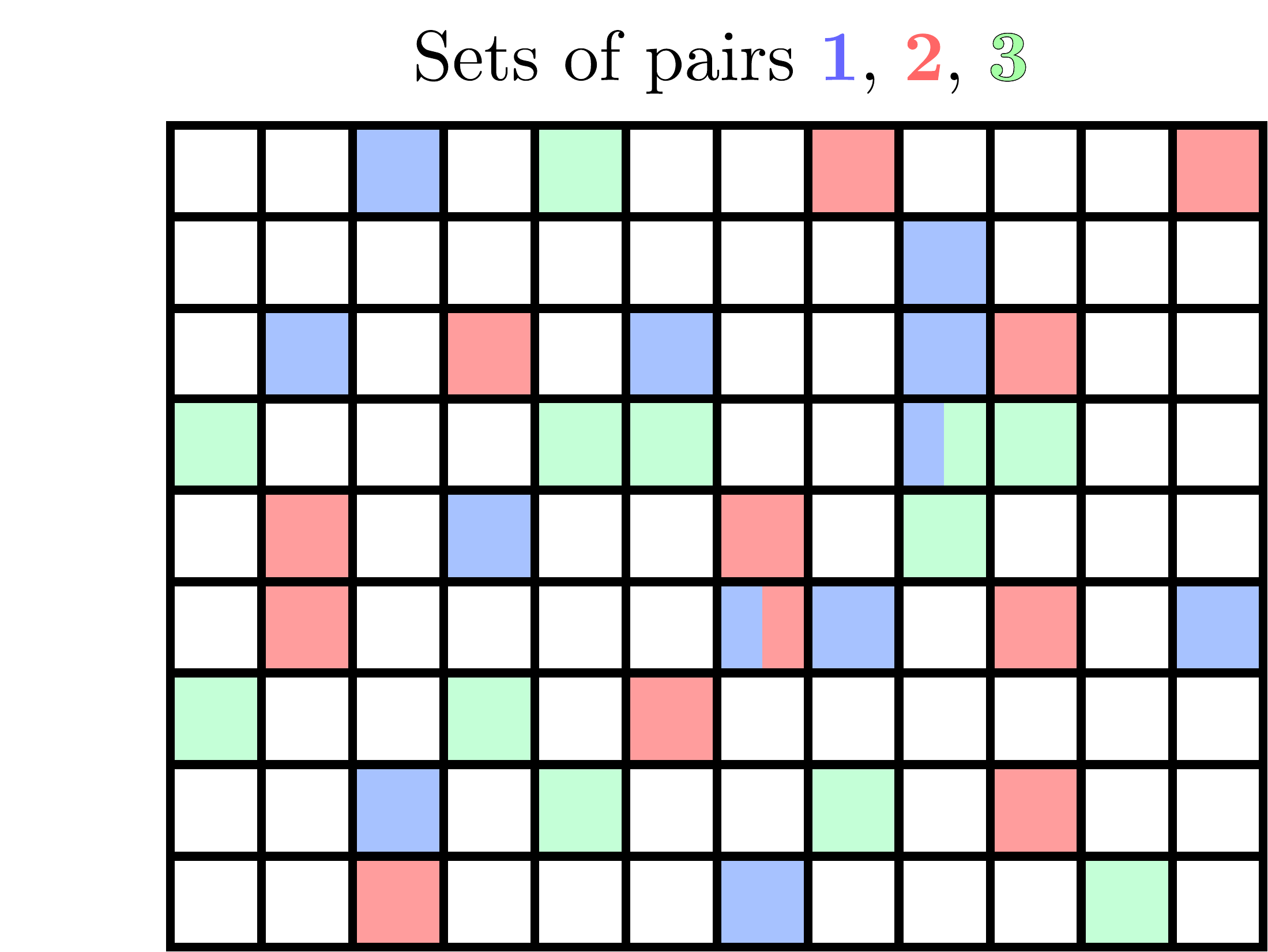}
  \caption{MoU based on randomized pairs}
  \label{fig:mou_incomp}
\end{subfigure}
\caption{Sampling strategies to build MoM and MoU, as well as  admitted proportion of outliers.}
\label{fig:blocks}
\end{figure*}

\subsection{Theoretical guarantees}\label{TG}

We now establish the statistical guarantees satisfied by the estimators introduced in Definition \ref{def:estimators} under Assumption \ref{ass:contamination}.
First notice that if $\hat{\mu}_\text{MoM}$ denotes with a language abuse the \emph{measure} such that for all application $\phi\colon \mathbb{R}^d \rightarrow \mathbb{R}$ it holds $\mathbb{E}_{\hat{\mu}_\text{MoM}}\left[\phi\right] = \text{MoM}_{\mathbf{X}}[\phi]$, it is direct to see that $\mathcal{W}_\text{MoM}(\hat{\mu}_n, \hat{\nu}_m) = \mathcal{W}(\hat{\mu}_\text{MoM}, \hat{\nu}_\text{MoM})$.
Then, it holds $\mathcal{W}_\text{MoM}(\hat{\mu}_n, \hat{\nu}_m) - \mathcal{W}(\mu, \nu) \le \mathcal{W}(\mu, \hat{\mu}_\text{MoM}) + \mathcal{W}(\hat{\nu}_\text{MoM}, \nu)$, and one may only focus on the theoretical guarantees of the right-hand side terms.
Before stating our main results, we need an additional assumption on the numbers of outliers $n_{\mathcal{O}}$ and $m_{\mathcal{O}}$, which are assumed to grow sub-linearly with respect to $n$ and $m$.
\begin{assumption}\label{ass:nO}
There exist $C_{\mathcal{O}} \geq 1$ and $0 \leq \alpha_{\mathcal{O}} <1 $ such that $n_{\mathcal{O}} \le C_{\mathcal{O}}^2~n^{\alpha_{\mathcal{O}}}$ and $m_{\mathcal{O}} \le C_{\mathcal{O}}^2~m^{\alpha_{\mathcal{O}}}$.
\end{assumption}

We start by an asymptotic result establishing the strong consistency of estimators in Definition \ref{def:estimators}.
It highlights the different outlier configurations allowed through conditions on the proportions of outliers $\tau_\mathbf{X}$ and $\tau_\mathbf{Y}$.

\begin{proposition}\label{prop:consistency}
Suppose that samples $\mathbf{X}$ and $\mathbf{Y}$ satisfy Assumptions \ref{ass:contamination} and \ref{ass:nO}.
Then, choosing $\KX = \lceil \sqrt{2\tau_\mathbf{X}}~n \rceil$, it holds:\vspace{-0.1cm}
$$
\mathcal{W}(\hat{\mu}_\mathrm{MoM}, \mu) \overset{a.s}{\longrightarrow} 0.
$$
If moreover $\tilde{\tau} \coloneqq \tau_\mathbf{X} + \tau_\mathbf{Y} - \tau_\mathbf{X}\tau_\mathbf{Y} < 1/2$, then choosing $\KX = \lceil \sqrt{2\tilde{\tau}}~n \rceil$ and $\KY = \lceil \sqrt{2\tilde{\tau}}~m \rceil$, it holds:
$$
\big \vert \mathcal{W}_{\mathrm{MoU}}(\hat{\mu}_n, \hnun) - \mathcal{W}(\mu, \nu) \big \vert \overset{a.s}{\longrightarrow} 0.
$$
If finally $\tau_\mathbf{X} + \tau_\mathbf{Y} < 1/2$ and $n=m$, then choosing $\KX = \KY = \lceil \sqrt{2(\tau_\mathbf{X} + \tau_\mathbf{Y})}~n \rceil$, it holds:
$$
\big \vert \mathcal{W}_{\mathrm{MoU-diag}}(\hat{\mu}_n, \hnun) - \mathcal{W}(\mu, \nu) \big \vert \overset{a.s}{\longrightarrow} 0.
$$
\end{proposition}

The key argument in the proof of Proposition \ref{prop:consistency} consists in converting the convergence of the different estimators into the convergences of blocks containing no outliers.
Numbers of blocks $\KX$ and $\KY$ are chosen such that (i)~such blocks are always in majority, and (ii) their sizes $n/\KX$ and $m/\KY$ go to infinity as $n$ and $m$ go to infinity.
Any other choice of $\KX$ and $\KY$ that satisfies this two conditions also ensures convergence.
If the outliers proportions are unknown, building $\KX$ and $\KY$ from upper bounds of $\tau_\mathbf{X}$ and $\tau_\mathbf{Y}$ thus does not impact Proposition \ref{prop:consistency}.
The conditions on $\KX$ and $\KY$ also constrain the proportions of outliers admitted, as illustrated in Figure \ref{fig:outliers}.
The assumption $n=m$ for $\mathcal{W}_{\mathrm{MoU-diag}}$ is necessary to be able to build a majority of sane blocks.
%
%
Our next proposition now investigates the nonasymptotic behavior of the proposed estimators.


\begin{proposition}\label{prop:appli_dev}
Suppose that samples $\mathbf{X}$ and $\mathbf{Y}$ satisfy Assumption \ref{ass:contamination}, and define $\Gamma\colon \tau \mapsto \sqrt{1 + \sqrt{2\tau}} / \sqrt{1 - 2\tau}$.
Then, for all $\delta \in ]0, \exp(-4n\sqrt{2\tau_{\mathbf{X}}})]$, choosing $\KX=\lceil \sqrt{2\tau_{\mathbf{X}}}~n\rceil$, it holds with probability at least $1 - \delta$:
\begin{equation*}
\mathcal{W}(\hat{\mu}_\mathrm{MoM}, \mu) \le \frac{C_1(\tau_{\mathbf{X}})}{n^{1/(d+2)}} + C_2(\tau_{\mathbf{X}}) \sqrt{\frac{\log(1/\delta)}{n}},
\end{equation*}
with $C_1(\tau) = 2 + C_L C_2(\tau)$, $C_2(\tau) = 4~\mathrm{diam}(\mathcal{K})~\Gamma(\tau)$, and $C_L$ a universal constant depending only on $\mathcal{B}_L$.

If furthermore $\tau_{\mathbf{X}} + \tau_{\mathbf{Y}} < 1/2$ and $n=m$, then for all $\delta \in ]0, \exp(-4n\sqrt{2(\tau_{\mathbf{X}} + \tau_{\mathbf{Y}})})]$, choosing $\KX = \KY = \lceil \sqrt{2(\tau_{\mathbf{X}}+\tau_{\mathbf{Y}})}~n\rceil$, it holds with probability at least $1 - \delta$:
\begin{gather*}
\Big| \mathcal{W}_\mathrm{MoU-diag}(\hat{\mu}_n, \hat{\nu}_m) - \mathcal{W}(\mu, \nu)\Big|\hspace{3cm}\\
\hspace{1cm}\le \frac{2C_1(\tau_{\mathbf{X}} + \tau_{\mathbf{Y}})}{n^{1/(d+2)}} + 2C_2(\tau_{\mathbf{X}} + \tau_{\mathbf{Y}}) \sqrt{\frac{\log(1/\delta)}{n}}.
\end{gather*}
\end{proposition}

The proof derives from concentration results established in \cite{papier2}, combined with a generic chaining argument.
It should be noticed that constant $C_2(\tau_\mathbf{X})$ explodes as $\tauX$ goes to $1/2$, which is expected: the more outliers, the more difficult it is to estimate $\mathcal{W}(\mu, \nu)$.
We also stress that the dependence in $1/\sqrt{1 - 2\tauX}$ is better than the $1/(1 - 2\tauX)^{3/2}$ term exhibited in \cite{monk}.
Integrating the deviation probabilities of Proposition \ref{prop:appli_dev} and using Assumption \ref{ass:nO}, we finally obtain our main theorem, that provides a nonasymptotic control on the expected value of our estimators deviations from $\mathcal{W}(\mu, \nu)$.

\begin{theorem}\label{thm:expectation}
Suppose that samples $\mathbf{X}$ and $\mathbf{Y}$ satisfy Assumptions \ref{ass:contamination} and \ref{ass:nO}, and recall the notation used in Proposition \ref{prop:appli_dev}.
Let $\beta \in [0, 1]$, then for all $n$ such that $n^{\frac{1}{d+2} + \frac{1 - \beta}{2}} \ge C_1(\tauX)/(2C_2(\tauX)(2\tauX)^\frac{1}{4})$, it holds:
\begin{equation*}
\mathbb{E}\left[\mathcal{W}(\hat{\mu}_\mathrm{MoM}, \mu)\right] ~\le~ \frac{\kappa_1(\tau_{\mathbf{X}})}{n^{1/(d+2)}} + \frac{\kappa_2(\tau_{\mathbf{X}})}{n^{(\beta- \alpha_{\mathcal{O}})/2}} + \frac{\kappa_3(\tau_{\mathbf{X}})}{n^{\beta/2}},
\end{equation*}
with $\kappa_1(\tau) = C_1(\tau)$, $\kappa_2(\tau) =2C_{\mathcal{O}}C_2(\tau)(2/\tau)^{1/4}$, and $\kappa_3(\tau) = \sqrt{\pi}C_2(\tau)/2$.

Of course, the above bound only makes sense if $\beta > \alpha_{\mathcal{O}}$.
In particular, if $\alpha_\mathcal{O} \le d / (d+2)$, setting $\beta=1$ gives that for all $n$ such that $n^{\frac{1}{d+2}} \ge C_1(\tauX)/(2C_2(\tauX)(2\tauX)^\frac{1}{4})$, with the notation $\kappa = \kappa_1 + \kappa_2 + \kappa_3$, it holds:
\begin{equation*}
\mathbb{E}\left[\mathcal{W}(\hat{\mu}_\mathrm{MoM}, \mu)\right] ~\le~ \kappa(\tau_{\mathbf{X}})~n^{-1/(d+2)}.
\end{equation*}
If furthermore $\tau_{\mathbf{X}} + \tau_{\mathbf{Y}} < 1/2$ and $n=m$, then for all $n$ s.t. $n^{\frac{1}{d+2}} \ge C_1(\tauX + \tauY)/(2C_2(\tauX + \tauY)(2(\tauX + \tauY))^\frac{1}{4})$, with the notation $\kappa' = 2\kappa_1 + 2\sqrt{2}\kappa_2 + 2\kappa_3$, it holds:
\begin{align*}
\mathbb{E}\big|\mathcal{W}_\mathrm{MoU-diag}(\hat{\mu}_n, \hat{\nu}_m) -& \mathcal{W}(\mu, \nu)\big|\\
&\le~ \kappa'(\tau_{\mathbf{X}} + \tau_{\mathbf{Y}})~n^{-1/(d+2)}.
\end{align*}\vspace{-0.88cm}
\end{theorem}

Theorem \ref{thm:expectation} highlights that the estimators proposed in Definition \ref{def:estimators} remarkably resist to the presence of outliers in the training datasets.
The price to pay is a slightly slower rate of order $O(n^{-1/(d+2)})$, that becomes equivalent in high dimension -- the usual setting of Optimal Transport -- to the standard $O(n^{-1/d})$ rate.
Interestingly, the dependence in the outliers growing rate $\alpha_\mathcal{O}$ is made explicit, and is in line with expectations (see below).
Unfortunately, the dependency between the blocks makes the nonasymptotic analysis harder for $\Wmoucb   $ and the computationally cheap randomized extensions discussed in Section \ref{sec:estimators}.
This theoretical challenge is left for future work.
We stress that there is no \emph{median-of-means miracle}.
If the number of blocks allows to cancel the outliers impact, the statistical performance then scales with the block size, \textit{i.e.} as $1/\sqrt{B_\mathbf{X}} = \sqrt{K_{\mathbf{X}}/n}$.
Since $K_{\mathbf{X}}$ is roughly $2 n_{\mathcal{O}}$, this means a $\sqrt{n_{\mathcal{O}}/n}$ rate.
So if one allows $n_{\mathcal{O}}$ to grow proportionally to $n$, the bound becomes vacuous.
To get guarantees improving with $n$, we thus need $n_{\mathcal{O}}$ to scale as $n^{\alpha_{\mathcal{O}}}$, for some $\alpha_{\mathcal{O}} < 1$, and the resulting rate is $n^{(1 - \alpha_{\mathcal{O}})/2}$, as found in Theorem \ref{thm:expectation}.
We finally point out that the condition on $n$ ensures $\mathcal{W}(\hat{\mu}_\mathrm{MoM}, \mu) \ge C_1(\tau_{\mathbf{X}})/n^{1/(d+2)} - C_2(\tau_{\mathbf{X}})\sqrt{\log(1/\delta)/n^\beta}$, as the right hand side is negative while $\mathcal{W}(\hat{\mu}_\mathrm{MoM}, \mu)$ is positive by construction.
A less stringent condition might be derived, using \textit{e.g.} the nature of the functions in $\mathcal{B}_L$.

\begin{remark}\label{rmk:other_IPMs}
The unique property of the Wasserstein distance we used, compared to other Integral Probability Metrics (IPMs) \cite{Bharath2012}, is the way to bound the entropy of the unit ball of Lipschitz functions. The present analysis can thus be extended in a direct fashion to any other IPM that has finite entropy.
\end{remark}

\section{MoM-based estimators in practice}\label{num}


 In this section, we first propose a novel algorithm to approximate the MoM/MoU-based estimators using neural networks and provide an empirical study of its behaviour on two toy datasets. Then, we show how to robustify Wasserstein-GANs and present MoMWGAN, a MoM-based variant of GAN, which is evaluated on two well-known image benchmarks.
\subsection{Approximation algorithm}

As show in Section ~\ref{MoM}, MoM/MoU-based estimation of the Wasserstein distance offers a robust alternative to the classical empirical estimator of $\mathcal{W}$. Indeed, the empirical estimator of $\mathcal{W}$ would not converge towards the target in the $\mathcal{O} \cup \mathcal{I}$ framework. The proposed estimators are consistent and have convergence rates of order $O(n^{-1/(d+2)})$ with the $\mathcal{O} \cup \mathcal{I}$ framework.These convergence rates are similar, when $d$ is not too small, to those of the empirical estimator of $\mathcal{W}$ in a non-contaminated setting. Nevertheless, the question of computing these estimators raises two major difficulties: (i) the optimization over the unit ball of Lipschitz functions is intractable, which is a difficulty common to the approximation of the standard Wasserstein distance, and (ii) the non-differentiability of the median-based loss. The first issue is well known of the practioners of the Wasserstein distance who usually prefer to rely on its primal definition with an entropy-based regularization \cite{cuturi13}. However, learning algorithms devoted to Wasserstein GANs overcome this by weight clipping \citep{arjovsky2017} or gradient penalization \citep{gulrajani2017improved} to impose to the GAN a Lipchitz constraint. Similarly we propose here to limit $\Phi$ to be a neural network with similar constraints on weights to ensure its $M$-Lipschitzianity. This enables to approximate the Wasserstein distance up to a (unknown) multiplicative coefficient $M$.
To overpass (ii), one can adopt MoM/MoU gradient descent. Exploited in the context of robust classification \citep{lecue2018robust}, using MoM/MoU gradient descent has been proved to be equivalent to minimize the expectation over the sampling strategy of blocks of $\Wmom,\Wmou$ and $\Wmoucb$.
Combining these techniques, we design novel algorithms to compute approximations of the proposed estimators: $\tWmom$ (see Algorithm \ref{algo::algo_WMOM}), $\tWmou$ and $\tWmoucb$ (see the Supplementary Material).

\begin{algorithm}[!h]
\caption{Approximation of $\mathcal{W}_{\mathrm{MoM}}(\X,\Y)$. }

\textit{Initialization:} $\eta$, the learning rate. $c$, the clipping parameter.  $w_0$, the initial weights. $\KX, \KY$ the number of blocks for $\X$ and $\Y$.
      \begin{algorithmic}[1]
          \For{$t=0,\ldots,n_{\text{iter}}$}

         \State Sample $\KX$ disjoint blocks $\mathcal{B}_1^{\X}, \ldots,\mathcal{B}_{\KX}^{\X}$ and $\KY$ disjoint blocks $\mathcal{B}_1^{\Y}, \ldots,\mathcal{B}_{\KY}^{\Y}$ from a sampling scheme
         \State Find both median blocks $\mathcal{B}_{med}^{\X}$ and $\mathcal{B}_{med}^{\Y}$

         \State \vspace*{-0.2cm} \small{\begin{align*} G_{w} \leftarrow  & \bigl\lfloor \KX / n \bigr\rfloor\hspace*{-0.25cm}  \underset{i \in \mathcal{B}_{med}^{\X}}{\sum}\hspace*{-0.2cm}  \nabla _{w}  \phi_w(X_i)\hspace*{-0.1cm} -\hspace*{-0.1cm} \bigl\lfloor\KY / m \bigr\rfloor \hspace*{-0.25cm}  \underset{j \in \mathcal{B}_{med}^{\Y}}{\sum} \hspace*{-0.2cm}  \nabla _{w}  \phi_w(Y_j))
         				\end{align*} }
         \State ~7.1~$w \leftarrow w + \eta \times \text{RMSProp}(w, G_w)$
         \State ~7.2~ $w \leftarrow \text{clip}(w, -c,c)$
\EndFor \\
 \textbf{Output}: $w,\; \tWmom, \; \phi_w$.

      \end{algorithmic}
      \label{algo::algo_WMOM}
\end{algorithm}

\subsection{Empirical study}

We denote $I_2$ the identity matrix of dimension $2$ and $\mathbf{v}$, the vector $(v,v)^\top$ with $v \in \mathbb{R}$.
\paragraph{Toy datasets.}
Two simulated datasets in 2D space with different kinds of anomalies are used in the experiments. The random vectors $X_1$ and $X_2$ are chosen to be distributed according a mixture of a standard Gaussian distribution  and an "anomaly" distribution, respectively $\bmA_1$ and $\bmA_2$ defined as follows. $\bmA_1$ is the uniform distribution $\mathcal{U}[\mathbf{-50}, \mathbf{50}]$ that mimics \textit{isolated outliers} while $\bmA_2$  is the standard Cauchy distribution shifted by 25, defined to mimic \textit{aggregate outliers} (see \textit{e.g.} \cite{Chandola}). The random vector $Y $ is Gaussian with $Y\sim \mathcal{N}(\mathbf{5},I_2)$,
Datasets $\bmD_1=(\X_1,\Y)$ and $\bmD_2=(\X_2,\Y)$ contain 500 independent and identical copies of $(X_1,Y)$, $(X_2,Y)$ respectively, with the same proportion of outliers $\tau_X$.
\paragraph{Evaluation metrics.}
The Lipchitz constant $M$ being unknown and highly depending of the clipping parameter choice, it wouldn't be appropriate to compare the true 1-Wasserstein value, equal to $\sqrt{50}$, with $\tWmom, \tWmou$ and $\tWmoucb$. Therefore, we propose to compare $\tWmom$, $\tWmou$ and $\tWmoucb$ to $\tW$, the 1-Wasserstein distance approximated by Algorithm \ref{algo::algo_WMOM}, when MoM is not used, \textit{e.g. } $\KX=\KY=1$, by measuring the absolute error between them.
\paragraph{Influence of $\KX,\KY$.}
The numbers of blocks, $\KX$ and $\KY$, are crucial parameters for computation. They define the trade-off between the robustness of the estimator and computational burden. However the theory does not give enough insights about their value: the necessary assumption for the consistency is only that they should be greater than $2\tau_X n$ (see Section \ref{TG}). An empirical study of the influence of their values on the behavior of the approximations of $\Wmom, \Wmou$ and $\Wmoucb$ is therefore much useful. For sake of simplicity, we set $\KX=\KY$ in the subsequent experiments.

In a first experiment, we explore the ability of \cref{algo::algo_WMOM} and variants described in the supplements to override outliers according to the values of $\KX$ and with different rates of outliers $\tauX$.  The approximations $\tWmom$,$\tWmou$ and $\tWmoucb$ are computed using a simple multi-layer perceptrons with one hidden layer and MoM gradient descent over several $\tauX$ and $\KX$ on both datasets. The experiment is repeated 20 times with various seeds. Mean results are displayed.  Figure ~\ref{fig:error} represents absolute deviations between the 1-Wasserstein distance approximated with a MLP when $\tauX=0$ and $\tWmou$ with various anomalies settings and different values of $K_X$. The reader is invited to refer to Section B of the supplements to see similar results for  $\tWmoucb$ and $\tWmom$). Shaded areas, in Figure ~\ref{fig:error} represent 25\%-75\% quantiles over the 20 repetitions. On both datasets, we observe that the approximation algorithm succeeds to provide an estimation of $\Wmou$, able to override outliers with different $\tau_{\X}$ while $\KX$ is high enough. From Section \ref{TG}, we know that $\KX$ needs to be higher than $2 \tauX n$ to have theoretical guarantees. Experiments show that in practice, this condition is not necessary in every situations. For example, when $\tau_{\X}=0.1$ (\textit{i.e.} 50 anomalies) in Figure ~\ref{fig:error} (left), only 70 blocks are needed to override outliers. The reason is that hypothesis makes things work in the worst case, \textit{i.e.}, when each outlier is isolated in one block which lead to have  $\tauX n$ contaminated blocks. This is rarely the case in practice, several blocks can be contaminated by  many outliers  and this is why fewer blocks are needed.
 \begin{figure}[!h]
\includegraphics[height=.13\textheight, trim=1cm .0cm 0cm 0cm,clip=true]{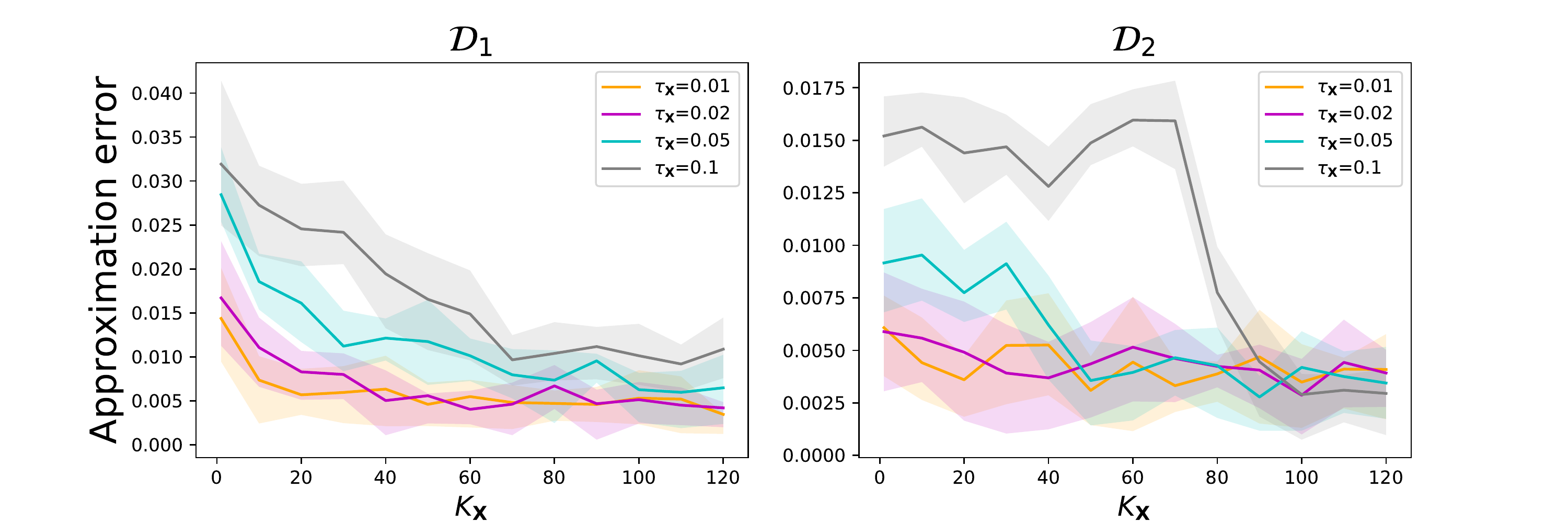}

\caption{$\tWmoudiag$  over $\KX$ for different  anomalies proportion $\tau_{\X}$ on $\mathcal{D}_1$ (left) and $\mathcal{D}_2$ (right).}
\label{fig:error}
\end{figure}

In a second experiment illustrated by Figure ~\ref{fig:CV_K}, we study the convergence of the approximation algorithm with and without anomalies for different values of $\KX$ on $\bmD_1$. To get a fair comparison between the different settings of the algorithm, we compare the predicted values across the "learning" epoch. Here during one epoch, the algorithm has made a gradient pass over the whole dataset, which means that one epoch corresponds one iteration of the approximation algorithm if $\KX=1$ (no MoM estimation),  and to  $\KX$ iterations, in the other cases. In both cases (with or without anomalies), the higher $\KX$ is, the faster the approximation algorithm converges. Without surprise, the MoM approach benefits from the same properties than a mini-batch approach. When there is no anomalies, the distance values reached after convergence are close to the "true" value (obtained with the plain estimator when $\KX=1$), especially when $\KX$ is lower. This means that the MoM-based algorithm can be used routinely instead of the plain estimator. With 5\% of anomalies,  one can see that distance values reached after convergence get closer to the target as $\KX$ grows.

 \begin{figure}[!h]
\includegraphics[height=.13\textheight, trim=1cm .0cm 0cm 0cm,clip=true]{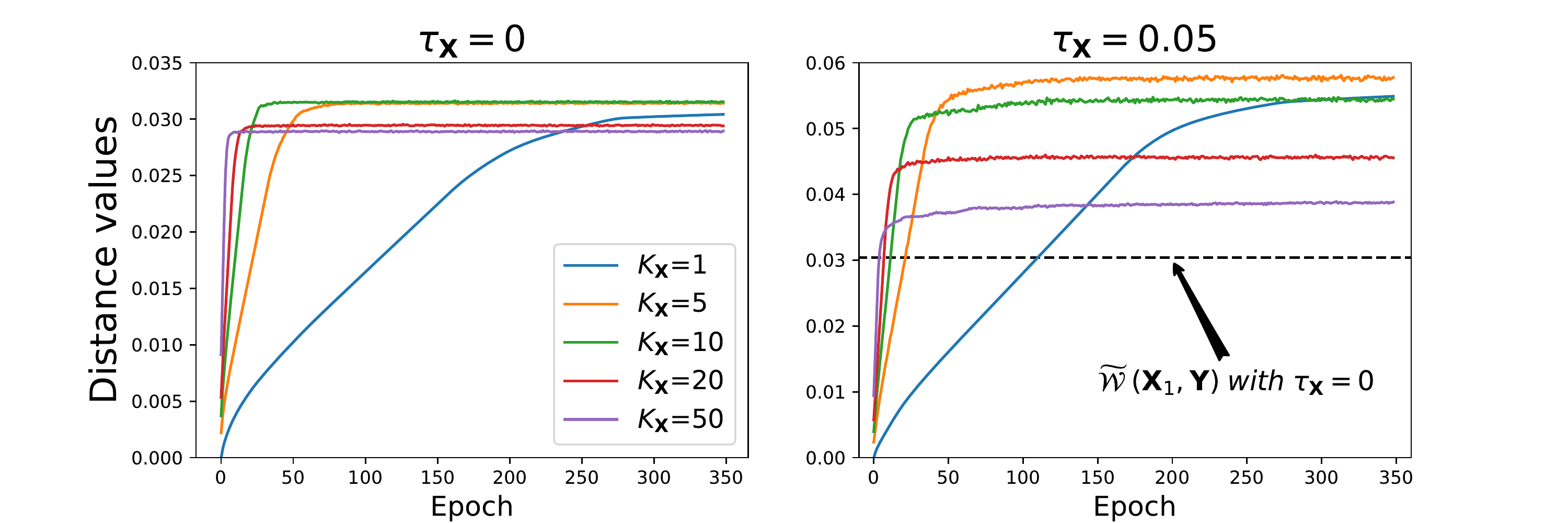}
\caption{Convergence of $\tWmoudiag$ without anomalies (left) and with 5\% anomalies (right) for different $\KX$.}
\label{fig:CV_K}
\end{figure}

\subsection{Application to robust Wasserstein GANs}\label{MoMalgo}

In this part, we introduce a robust modification of WGANs, named MoMWGAN, using one of the three proposed estimators in Section \ref{MoM}.
\begin{figure*}[!h]
\begin{center}
\setlength{\tabcolsep}{2em} 
{\renewcommand{\arraystretch}{2}
\begin{tabular}{cccc}
\multicolumn{2}{c}{WGAN}&  \multicolumn{2}{c}{MoMWGAN}  \\
\includegraphics[trim=2cm 0 .5cm 0, scale=0.35]{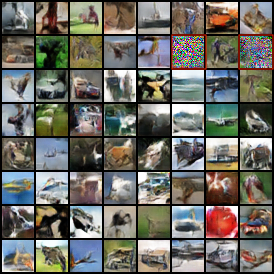}&
 \includegraphics[trim=3.2cm 0 .5cm 0, scale=0.35]{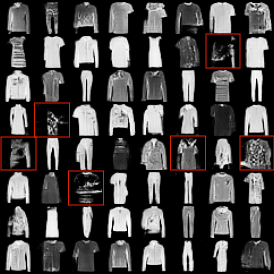}
&
\includegraphics[trim=3.2cm 0 .5cm 0,scale=0.35]{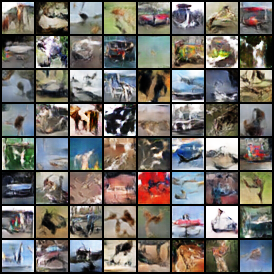}&
\includegraphics[trim=3.2cm 0 .5cm 0,scale=0.35]{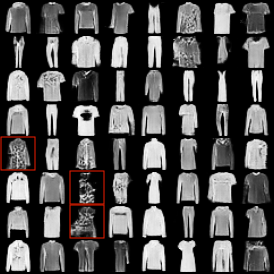}
\\
\end{tabular}}
\end{center}
\caption{Generated samples from trained WGAN and MoMWGAN on CIFAR10 and Fashion MNIST datasets.}\label{images}
\end{figure*}
The behaviour of likelihood-free generative modeling such as Generative Adversarial Networks in the presence of outliers, \textit{i.e.}, with heavy-tails distributions  or contaminated data, has been poorly investigated up to very recently. At our knowledge, the unique reference is \citep{gao2018robust}.
In particular, \citet{gao2018robust} have studied theoretically and empirically the robustness of f-GAN in the special case of  mean estimation for elliptical distributions. In contrast, we illustrate here the theoretical results of \cref{MoM} by applying a MoM approach to robustify WassersteinGAN and show on two real-world image benchmarks how this new variant of GAN behaves when learned with contaminated data.

\noindent \textbf{Reminder on GAN:}
Let us briefly recall the GAN principle. A GAN learns a function $g_{\theta}:\mathcal{Z} \rightarrow \mathcal{X}$ such that samples generate by $g_{\theta}(z) \sim P_{\theta}$, taking as input a sample $z$ (from some reference measure $\xi$, often Gaussian) in a latent space $ \mathcal{Z}$, are close to those of the true distribution $P_r$ of data. 
Wasserstein GANs \citep{arjovsky2017,gulrajani2017improved} use the 1-Wasserstein Distance under its Kantorovich-Rubinstein dual formula as the loss function. Instead of maximizing over the unit ball of Lipschitz functions, one uses a parametric family of M-Lipschitz functions under the form of neural net with clipped weights $w$ \citep{arjovsky2017}. Following up the theoretical analysis of Section \ref{MoM}, we introduce a MoM-based WGAN (MoMWGAN) model, combining the $\Wmom$ estimator studied in \ref{MoM} and WGAN's framework.
Following the weight clipping approach, MoMWGAN boils down to the problem:
\begin{equation*}\label{eq:MoMWGAN-loss}
\underset{ \theta }{\min } \; \;\underset{w}{\max } \; \;  \Bigl\{ \text{MoM}_{\X} [f_w] - \frac{1}{m} \sum_{j=1}^{m} f_w(g_{\theta}(Z_j)), ~ k\leq \KX\Bigr\}
\end{equation*}
Note that the MoM procedure is chosen to be only applied on the observed contaminated sample. It is not clear in which way the sample drawn from the currently learned density is polluted and thus defining the number of blocks would be an issue.
Optimization in WGAN is usually performed by taking mini-batches to reduce the computational load. In the same spirit, we apply MoM inside contaminated mini-batches as described in Algorithm \ref{MWGAN}. To get the outliers-robust property observed in the numerical experiments, we pay the price of finding the median block at each step by evaluating the loss which significantly increases the computational complexity.

\begin{algorithm}[H]
\caption{MoMWGAN}

\textit{Initialization:} $\eta$, the learning rate. $c$, the clipping parameter. $b$, the batch size. $n_c$, the number of critic iterations per generator iteration, $\KX$ the number of blocks. $w_0,\theta_0$ the initial critic/generator's parameters.
      \begin{algorithmic}[1]
         \While{$\theta$ has not converged}
          \For{$t=0,\ldots,n_c$}
        \State Sample $\{X_i \}_{i=1}^b \sim P_r$ to get $\X_t$ and sample $\{z_i \}_{i=1}^b \sim \xi$ to get $\mathbf{Z}_t$

         \vspace*{0.1cm}
         \State Updating $w$ with step 2-6 of Algorithm \ref{algo::algo_WMOM} with $\X=\X_t$ and $\Y=g_{\theta}(\mathbf{Z}_t)$
\EndFor
              \vspace*{0.1cm}
   \State       Sample $\{ Z_j \}_{j=1}^b \sim \xi$
       \vspace*{0.1cm}
       \State ~$g_{\theta} \leftarrow - \nabla_{\theta} \frac{1}{b} \sum_{j=1}^{b} f_w(g_{\theta}(Z_j))$
\State ~$\theta \leftarrow \theta - \eta \times \text{RMSProp}(\theta, g_{\theta})$
            \vspace*{0.1cm}
         \EndWhile

      \end{algorithmic}
      \label{MWGAN}
\end{algorithm}

\paragraph{Numerical experiments}
To test the robustness of MoMWGAN we contaminated two well-known image datasets, CIFAR10 and Fashion MNIST, with two anomalies settings.  \textit{Noise} based-anomalies are added to CIFAR10, \textit{i.e.}, images with random intensity pixels drawn from a uniform law. For Fashion MNIST, the five first classes are considered as "informative data" while the sixth (Sandal) contains the anomalies. In both settings,
WGAN and MoMWGAN are trained on the training samples contaminated in a uniform fashion with a proportion of 1.5\% of outliers in both datasets. Both models use standard parameters of WGAN. $\KX=4$ blocks have been used by MoMWGAN in both experiments. To assess performance of the resulting GANs, we generated 50000 generated images using each model (WGAN and MoMGAN) and measured the Fr\'echet Inception Distance (FID) \citep{Heusel} between the generated examples  in both cases and the (real) test sample.
  Table 1 shows that MoMWGAN improves upon WGAN in terms of outliers-robustness. Furthermore, some generated images are represented in Figure \ref{images}. One can see that outliers do not affect MoMWGAN generated samples while WGAN reproduce noise on contaminated CIFAR10 dataset. For Fashion MNIST, one may see that fewer images are degraded with MoMWGAN generator.

\setlength{\tabcolsep}{0.5em} 
{\renewcommand{\arraystretch}{1}
\begin{tabular}{ccc}
&WGAN& MoMWGAN\\
\hline
Polluted CIFAR10 & 57 & 55.9 \\
Polluted Fashion MNIST & 13.8 & 13.2 \\ \hline
\end{tabular}}
\captionof{table}{FID on polluted datasets.}

\section{Conclusion and perspectives} \label{conclusion}

In this paper, we have introduced three robust estimators of the Wasserstein distance based on MoM methodology.
We have shown asymptotic and non-asymptotic results in the context of polluted data, \textit{i.e.} the $\mathcal{O}\cup \mathcal{I}$ framework. Surpassing computational issues, we have designed an algorithm to compute, in a efficient way, these estimators. Numerical experiments have highlighted the behavior of these estimators over their unique tuning parameter.
Finally, we proposed to robustify WGANs using one of the introduced estimators and have shown  its benefits  on convincing numerical results.
The theoretically well-founded MoM approaches to robustify the Wasserstein distance open the door to numerous applications beyond WGAN, including variational generative modeling. The promising MoMGAN deserves more attention and future work will concern the analysis of the estimator it provides.


\section*{ Acknowlegments}

The authors thank Pierre Colombo for his helpful remarks. This work has been funded by BPI France in the context of the PSPC Project Expresso (2017-2021).

\bibliographystyle{unsrtnat}
\bibliography{Robust_Wasserstein}

\begin{thebibliography}{45}
\providecommand{\natexlab}[1]{#1}
\providecommand{\url}[1]{\texttt{#1}}
\expandafter\ifx\csname urlstyle\endcsname\relax
  \providecommand{\doi}[1]{doi: #1}\else
  \providecommand{\doi}{doi: \begingroup \urlstyle{rm}\Url}\fi

\bibitem[Villani(2003)]{Villani}
Cedric Villani.
\newblock \emph{Topics in Optimal Transportation}.
\newblock Graduate Studies in Mathematics Series. American Mathematical
  Society, New York, 2003.

\bibitem[Santambrogio(2015)]{Santambrogio}
Filippo Santambrogio.
\newblock \emph{Optimal Transport for Applied Mathematicians}.
\newblock Birkhauser, 2015.

\bibitem[Peyré and Cuturi(2019)]{Peyre}
Gabriel Peyré and Marco Cuturi.
\newblock Computational optimal transport.
\newblock \emph{Foundations and Trends® in Machine Learning}, 11\penalty0
  (5-6):\penalty0 355--607, 2019.
\newblock URL \url{http://dx.doi.org/10.1561/2200000073}.

\bibitem[Csiszàr(1963)]{csiszar}
I.~Csiszàr.
\newblock Eine informationstheoretische ungleichung und ihre anwendung auf den
  bewis der ergodizität von markhoffschen kette.
\newblock \emph{Magyer Tud. Akad. Mat. Kutato Int. Koezl}, 8:\penalty0 85--108,
  1963.

\bibitem[Nguyen et~al.(2009)Nguyen, Wainwright, and Jordan]{nguyen2009}
XuanLong Nguyen, Martin~J. Wainwright, and Michael~I. Jordan.
\newblock On surrogate loss functions and f -divergences.
\newblock \emph{Ann. Statist.}, 37\penalty0 (2):\penalty0 876--904, 04 2009.
\newblock \doi{10.1214/08-AOS595}.
\newblock URL \url{https://doi.org/10.1214/08-AOS595}.

\bibitem[Goodfellow et~al.(2014)Goodfellow, Pouget-Abadie, Mirza, Xu,
  Warde-Farley, Ozair, Courville, and Bengio]{goodfellow}
Ian Goodfellow, Jean Pouget-Abadie, Mehdi Mirza, Bing Xu, David Warde-Farley,
  Sherjil Ozair, Aaron Courville, and Yoshua Bengio.
\newblock Generative adversarial nets.
\newblock In \emph{Advances in Neural Information Processing Systems (NeurIPS
  2014)}, 2014.

\bibitem[Arjovsky et~al.(2017)Arjovsky, Chintala, and Bottou]{arjovsky2017}
Martin Arjovsky, Soumith Chintala, and Léon Bottou.
\newblock Wasserstein gan, 2017.

\bibitem[Gulrajani et~al.(2017)Gulrajani, Ahmed, Arjovsky, Dumoulin, and
  Courville]{gulrajani2017improved}
Ishaan Gulrajani, Faruk Ahmed, Martin Arjovsky, Vincent Dumoulin, and Aaron
  Courville.
\newblock Improved training of wasserstein gans, 2017.

\bibitem[Bousquet et~al.(2017)Bousquet, Gelly, Tolstikhin, Simon-Gabriel, and
  Sch{\"o}lkopf]{bousquetetal17}
Olivier Bousquet, Sylvain Gelly, Ilya Tolstikhin, Carl-Johann Simon-Gabriel,
  and Bernhard Sch{\"o}lkopf.
\newblock From optimal transport to generative modeling: the vegan cookbook.
\newblock \emph{arXiv preprint arXiv:1705.07642}, 2017.

\bibitem[{Courty} et~al.(2017){Courty}, {Flamary}, {Tuia}, and
  {Rakotomamonjy}]{courty2017}
N.~{Courty}, R.~{Flamary}, D.~{Tuia}, and A.~{Rakotomamonjy}.
\newblock Optimal transport for domain adaptation.
\newblock \emph{IEEE Transactions on Pattern Analysis and Machine
  Intelligence}, 39\penalty0 (9):\penalty0 1853--1865, 2017.

\bibitem[Flamary et~al.(2018)Flamary, Cuturi, Courty, and
  Rakotomamonjy]{flamary18}
R{\'{e}}mi Flamary, Marco Cuturi, Nicolas Courty, and Alain Rakotomamonjy.
\newblock Wasserstein discriminant analysis.
\newblock \emph{Mach. Learn.}, 107\penalty0 (12):\penalty0 1923--1945, 2018.

\bibitem[Genevay et~al.(2018)Genevay, Peyre, and Cuturi]{genevay18}
Aude Genevay, Gabriel Peyre, and Marco Cuturi.
\newblock Learning generative models with sinkhorn divergences.
\newblock In \emph{Proceedings of the 21st International Conference on
  Artificial Intelligence and Statistics (AISTATS 2018)}, 2018.

\bibitem[Cuturi et~al.(2013)Cuturi, Teboul, and Vert]{cuturi13}
Marco Cuturi, Olivier Teboul, and Jean-Philippe Vert.
\newblock Sinkhorn distances: Lightspeed computation of optimal transportation.
\newblock In \emph{Advances in Neural Information Processing Systems (NeurIPS
  2013)}, 2013.

\bibitem[Dudley(1969)]{dudley1969}
R.~M. Dudley.
\newblock The speed of mean glivenko-cantelli convergence.
\newblock \emph{Ann. Math. Statist.}, 40\penalty0 (1):\penalty0 40--50, 02
  1969.

\bibitem[Bassetti et~al.(2006)Bassetti, Bodini, and Regazzini]{bassetti2006}
Federico Bassetti, Antonella Bodini, and Eugenio Regazzini.
\newblock On minimum kantorovich distance estimators.
\newblock \emph{Statistics and Probability Letters}, 76:\penalty0 1298--1302,
  07 2006.

\bibitem[Weed and Bach(2019)]{weed2019}
Jonathan Weed and Francis Bach.
\newblock Sharp asymptotic and finite-sample rates of convergence of empirical
  measures in wasserstein distance.
\newblock \emph{Bernoulli}, 25\penalty0 (4A):\penalty0 2620--2648, 11 2019.

\bibitem[Gao et~al.(2018)Gao, Liu, Yao, and Zhu]{gao2018robust}
Chao Gao, Jiyi Liu, Yuan Yao, and Weizhi Zhu.
\newblock Robust estimation and generative adversarial nets, 2018.

\bibitem[Futami et~al.(2018)Futami, Sato, and Sugiyama]{futami}
Futoshi Futami, Issei Sato, and Masashi Sugiyama.
\newblock Variational inference based on robust divergences.
\newblock In \emph{Proceedings of the 21st International Conference on
  Artificial Intelligence and Statistics (AISTATS 2018).}, 2018.

\bibitem[Huber and Ronchetti(2009)]{Huber2009}
Peter~J. Huber and Elvezio~M. Ronchetti.
\newblock \emph{Robust Statistics (Second Edition)}.
\newblock John Wiley \& Sons, Inc., Hoboken, New Jersey, 2009.

\bibitem[Nemirovsky and Yudin(1983)]{nemirovsky1983problem}
Arkadii~Semenovich Nemirovsky and David~Borisovich Yudin.
\newblock \emph{Problem Complexity and Method Efficiency in Optimization}.
\newblock John Wiley \& Sons Ltd, 1983.

\bibitem[Jerrum et~al.(1986)Jerrum, Valiant, and Vazirani]{jerrum1986random}
Mark~R Jerrum, Leslie~G Valiant, and Vijay~V Vazirani.
\newblock Random generation of combinatorial structures from a uniform
  distribution.
\newblock \emph{Theoretical Computer Science}, 43:\penalty0 169--188, 1986.

\bibitem[Alon et~al.(1999)Alon, Matias, and Szegedy]{alon1999space}
Noga Alon, Yossi Matias, and Mario Szegedy.
\newblock The space complexity of approximating the frequency moments.
\newblock \emph{Journal of Computer and system sciences}, 58\penalty0
  (1):\penalty0 137--147, 1999.

\bibitem[Catoni(2012)]{catoni2012challenging}
Olivier Catoni.
\newblock Challenging the empirical mean and empirical variance: a deviation
  study.
\newblock In \emph{Annales de l'Institut Henri Poincar{\'e}, Probabilit{\'e}s
  et Statistiques}, volume~48, pages 1148--1185. Institut Henri Poincar{\'e},
  2012.

\bibitem[Devroye et~al.(2016)Devroye, Lerasle, Lugosi, Oliveira,
  et~al.]{devroye2016sub}
Luc Devroye, Matthieu Lerasle, Gabor Lugosi, Roberto~I Oliveira, et~al.
\newblock Sub-gaussian mean estimators.
\newblock \emph{The Annals of Statistics}, 44\penalty0 (6):\penalty0
  2695--2725, 2016.

\bibitem[Minsker et~al.(2015)]{minsker2015geometric}
Stanislav Minsker et~al.
\newblock Geometric median and robust estimation in banach spaces.
\newblock \emph{Bernoulli}, 21\penalty0 (4):\penalty0 2308--2335, 2015.

\bibitem[Hsu and Sabato(2016)]{hsu2016loss}
Daniel Hsu and Sivan Sabato.
\newblock Loss minimization and parameter estimation with heavy tails.
\newblock \emph{The Journal of Machine Learning Research}, 17\penalty0
  (1):\penalty0 543--582, 2016.

\bibitem[Lugosi and Mendelson(2017)]{lugosi2017sub}
G{\'a}bor Lugosi and Shahar Mendelson.
\newblock Sub-gaussian estimators of the mean of a random vector.
\newblock \emph{arXiv preprint arXiv:1702.00482}, 2017.

\bibitem[Joly and Lugosi(2016)]{joly2016robust}
Emilien Joly and G{\'a}bor Lugosi.
\newblock Robust estimation of u-statistics.
\newblock \emph{Stochastic Processes and their Applications}, 126\penalty0
  (12):\penalty0 3760--3773, 2016.

\bibitem[Laforgue et~al.(2019)Laforgue, Cl\'{e}men\c{c}on, and
  Bertail]{laforgue2019medians}
Pierre Laforgue, Stephan Cl\'{e}men\c{c}on, and Patrice Bertail.
\newblock On medians of (randomized) pairwise means.
\newblock In \emph{Proceedings of the 36th International Conference on Machine
  Learning (ICML 2019)}, 2019.

\bibitem[Bubeck et~al.(2013)Bubeck, Cesa-Bianchi, and
  Lugosi]{bubeck2013bandits}
S{\'e}bastien Bubeck, Nicolo Cesa-Bianchi, and G{\'a}bor Lugosi.
\newblock Bandits with heavy tail.
\newblock \emph{IEEE Transactions on Information Theory}, 59\penalty0
  (11):\penalty0 7711--7717, 2013.

\bibitem[Lugosi and Mendelson(2019)]{lugosi2019risk}
Gabor Lugosi and Shahar Mendelson.
\newblock Risk minimization by median-of-means tournaments.
\newblock \emph{Journal of the European Mathematical Society}, 2019.

\bibitem[Depersin and Lecu{\'e}(2019)]{depersin2019robust}
Jules Depersin and Guillaume Lecu{\'e}.
\newblock Robust subgaussian estimation of a mean vector in nearly linear time.
\newblock \emph{arXiv preprint arXiv:1906.03058}, 2019.

\bibitem[Laforgue et~al.(2020)Laforgue, Staerman, and
  Cl\'{e}men\c{c}on]{papier2}
P.~Laforgue, G.~Staerman, and S.~Cl\'{e}men\c{c}on.
\newblock How robust is the median-of-means? concentration bounds in presence
  of outliers.
\newblock arxiv.org/abs/2006.05240, 2020.

\bibitem[Lerasle et~al.(2019)Lerasle, Szabo, Mathieu, and Lecu{\'e}]{monk}
Matthieu Lerasle, Zoltan Szabo, Timoth{\'e}e Mathieu, and Guillaume Lecu{\'e}.
\newblock Monk -- outlier-robust mean embedding estimation by median-of-means.
\newblock In \emph{Proceedings of the 36th International Conference on Machine
  Learning (ICML 2019)}, 2019.

\bibitem[Lecu{\'e} et~al.(2018)Lecu{\'e}, Lerasle, and
  Mathieu]{lecue2018robust}
Guillaume Lecu{\'e}, Matthieu Lerasle, and Timoth{\'e}e Mathieu.
\newblock Robust classification via mom minimization.
\newblock \emph{arXiv preprint arXiv:1808.03106}, 2018.

\bibitem[Kantorovich and Rubinstein(1958)]{kantorovich1958}
Leonid~Vasilevich Kantorovich and Gennady~S Rubinstein.
\newblock On a space of completely additive functions.
\newblock \emph{Vestnik Leningrad. Univ}, 13\penalty0 (7):\penalty0 52--59,
  1958.

\bibitem[Boissard(2011)]{boissard2011}
Emmanuel Boissard.
\newblock Simple bounds for the convergence of empirical and occupation
  measures in 1-wasserstein distance.
\newblock \emph{Electron. J. Probab.}, 16(83):\penalty0 2296--2333, 2011.

\bibitem[Fournier and Guillin(2015)]{FG15}
Nicolas Fournier and Arnaud Guillin.
\newblock \emph{Probability Theory and Related Fields}, 162\penalty0
  (3):\penalty0 707--738, 2015.

\bibitem[Lee(1990)]{Lee90}
A.~J. Lee.
\newblock \emph{${U}$-statistics: Theory and practice}.
\newblock Marcel Dekker, Inc., New York, 1990.

\bibitem[Lecué and Lerasle(2020)]{lecue2017robust}
Guillaume Lecué and Matthieu Lerasle.
\newblock Robust machine learning by median-of-means: Theory and practice.
\newblock \emph{Ann. Statist.}, 48\penalty0 (2):\penalty0 906--931, 04 2020.
\newblock \doi{10.1214/19-AOS1828}.
\newblock URL \url{https://doi.org/10.1214/19-AOS1828}.

\bibitem[Sriperumbudur et~al.(2012)Sriperumbudur, Fukumizu, Gretton,
  Schölkopf, and Lanckriet]{Bharath2012}
Bharath~K. Sriperumbudur, Kenji Fukumizu, Arthur Gretton, Bernhard Schölkopf,
  and Gert R.~G. Lanckriet.
\newblock On the empirical estimation of integral probability metrics.
\newblock \emph{Electron. J. Statist.}, 6:\penalty0 1550--1599, 2012.

\bibitem[Chandola et~al.(2009)Chandola, Banerjee, and Kumar]{Chandola}
Varun Chandola, Arindam Banerjee, and Vipin Kumar.
\newblock Anomaly detection: A survey.
\newblock \emph{ACM Comput. Surv.}, 41\penalty0 (3):\penalty0 15:1--15:58,
  2009.
\newblock ISSN 0360-0300.

\bibitem[Heusel et~al.(2017)Heusel, Ramsauer, Unterthiner, Nessler, and
  Hochreiter]{Heusel}
Martin Heusel, Hubert Ramsauer, Thomas Unterthiner, Bernhard Nessler, and Sepp
  Hochreiter.
\newblock Gans trained by a two time-scale update rule converge to a local nash
  equilibrium.
\newblock In \emph{Advances in Neural Information Processing Systems 30}, pages
  6626--6637. 2017.

\bibitem[van~de Geer(2000)]{vdG00}
S.~van~de Geer.
\newblock \emph{Empirical Processes in {M}-Estimation}.
\newblock Cambridge University Press, 2000.

\bibitem[Kolmogorov and Tihomirov(1961)]{kolmogorovtihomirov}
A.~N. Kolmogorov and V.~M. Tihomirov.
\newblock On the empirical estimation of integral probability metrics.
\newblock \emph{American Mathematical Society Translations 2}, 17:\penalty0
  277--364, 1961.

\end{thebibliography}

\clearpage

\setcounter{section}{0}
\renewcommand{\thesection}{\Alph{section}}

\onecolumn
 { \Large \bf  Supplementary Material to the Article: \\
 When OT meets MoM: Robust estimation of Wasserstein Distance}

\vspace*{0.5cm}

\section{Technical Proofs }

In this section are detailed the proofs of the theoretical claims stated in the core article. We first recall a simple lemma on the difference between two median vectors.

\begin{lemme}\label{lem:median}
Let $\bm{a}$ and $\bm{b}$ be two vectors of $\mathbb{R}^d$.
Then it holds
\begin{equation*}
\big| \mathrm{median}(\bm{a}) - \mathrm{median}(\bm{b}) \big| \le \|\bm{a} - \bm{b}\|_\infty.
\end{equation*}
\end{lemme}

\begin{proof}
It is direct to see that:
\begin{equation*}
\bm{a} \preceq \bm{b} \preceq \bm{c} ~~\Rightarrow~~ \mathrm{median}(\bm{a}) \le \mathrm{median}(\bm{b}) \le \mathrm{median}(\bm{c}).
\end{equation*}

Thus, for all $\bm{b}$ within the infinite ball of center $\bm{a}$ and radius $\epsilon$ it holds:
\begin{equation*}
\mathrm{median}(\bm{a}) - \epsilon = \mathrm{median}(\bm{a} -  \epsilon \bm{1}_d) \le \mathrm{median}(\bm{b}) \le \mathrm{median}(\bm{a} +  \epsilon \bm{1}_d) = \mathrm{median}(\bm{a}) + \epsilon.
\end{equation*}
Hence the conclusion.
\end{proof}


\subsection{Proof of Proposition 4}

We first show the consistency of $\Wmoucb(\hat{\mu}_n, \hat{\nu}_m)$, that of $\mathcal{W}(\hat{\mu}_\mathrm{MoM}, \mu)$ and $\Wmou(\hat{\mu}_n, \hat{\nu}_m)$ being then straightforward adaptations. 
 Assume that $\tilde{\tau}=\tau_{\mathbf{X}}  + \tau_{\mathbf{Y}}  - \tau_{\mathbf{X}}  \tau_{\mathbf{Y}}  < 1/2$, and $K_{\mathbf{X}} , K_{\mathbf{Y}} > 0$ such that $2(\tau_{\mathbf{X}}  + \tau_{\mathbf{Y}}  - \tau_{\mathbf{X}}  \tau_{\mathbf{Y}} ) < K_{\mathbf{X}}  K_{\mathbf{Y}}  / (nm)$.
The latter condition implies that the blocks containing no outlier are in majority.
Indeed, the number of contaminated blocks is upper bounded~by:
\begin{equation*}
n_\mathcal{O} K_{\mathbf{Y}}  + n_\mathcal{O} K_{\mathbf{X}}  - n_\mathcal{O}  n_\mathcal{O} \le (\tau_{\mathbf{X}}  + \tau_{\mathbf{Y}}  - \tau_{\mathbf{X}} \tau_{\mathbf{Y}} )nm < K_{\mathbf{X}}  K_{\mathbf{Y}} /2.
\end{equation*}
One may choose $\KX$ and $\KY$ the lower as possible such that the above condition is respected. Following this, it is a natural choice to set $\KX = \lceil \sqrt{2\tilde{\tau}}~n \rceil$ and $\KY = \lceil \sqrt{2\tilde{\tau}}~m \rceil$. 

Let $\mathcal{I}_{\mathbf{X}} $ (respectively $\mathcal{I}_{\mathbf{Y}} $) denote the set of indices of \textbf{X} blocks (respectively \textbf{Y} blocks) containing no outlier.
Let $\mathcal{K}$ be a bounded subspace of $\mathbb{R}^d$, and assume that $X, Y$ are valued in $\mathcal{X}, \mathcal{Y} \subset \mathcal{K}$.
Finally, we denote by $\overline{\phi}_{\mathbf{X} ,k}$ and $\overline{\phi}_{\mathbf{Y},l}$ the quantities
\begin{equation*}
\overline{\phi}_{\mathbf{X} ,k} = \dfrac{1}{\BX}\sum_{i \in \mathcal{B}^{\mathbf{X}}_k}\phi(X_i), \qquad \text{and} \qquad  \overline{\phi}_{\mathbf{Y},l} = \dfrac{1}{\BY}\sum_{j \in \mathcal{B}^{\mathbf{Y}}_l}\phi(Y_j).
\end{equation*}

Using the shortcut notation $\mathbb{E}_\mu\left[\phi\right] = \mathbb{E}_{X \sim \mu}\left[\phi(X)\right]$ and $\mathbb{E}_\nu\left[\phi\right] = \mathbb{E}_{Y \sim \nu}\left[\phi(Y)\right]$, first notice that:
\begin{align}\label{eq:major_1}
\Wmoucb(\hat{\mu}_n, \hat{\nu}_m) &= \underset{\phi \in \mathcal{B}_L}{\sup } ~~ \text{MoU}_{\mathbf{XY}}[h_{\phi}],\nonumber\\
&= \underset{\phi \in \mathcal{B}_L}{\sup } ~~ \underset{\substack{1 \le k \le K_{\mathbf{X}}\\[0.1cm]1 \le l \le K_{\mathbf{Y}}}}{\text{med}}\Big\{ \overline{\phi}_{\mathbf{X},k} - \overline{\phi}_{\mathbf{Y},l}\Big\},\nonumber\\
&= \underset{\phi \in \mathcal{B}_L}{\sup } ~~ \underset{\substack{1 \le k \le K_{\mathbf{X}}\\[0.1cm]1 \le l \le K_{\mathbf{Y}}}}{\text{med}}\Big\{ \overline{\phi}_{\mathbf{X},k}-\mathbb{E}_\mu[\phi] +\mathbb{E}_\mu[\phi] -\mathbb{E}_\nu[\phi] + \mathbb{E}_\nu[\phi] - \overline{\phi}_{\mathbf{Y},l}\Big\},\nonumber\\
& \leq \underset{\phi \in \mathcal{B}_L}{\sup} ~~ \underset{\substack{1 \le k \le K_{\mathbf{X}}\\[0.1cm]1 \le l \le K_{\mathbf{Y}}}}{\text{med}}\Big\{ \overline{\phi}_{\mathbf{X},k} - \mathbb{E}_\mu[\phi]  +\mathbb{E}_\nu[\phi] - \overline{\phi}_{\mathbf{Y},l} \Big\}+\mathcal{W}(\mu, \nu).
\end{align}

Conversely, it holds:
\begin{align}\label{eq:major_2}
\mathcal{W}(\mu, \nu) &= \underset{\phi \in \mathcal{B}_L}{\sup} \big\{\mathbb{E}_\mu\left[\phi\right] - \mathbb{E}_\nu\left[\phi\right]\big\},\nonumber\\
&\le \underset{\phi \in \mathcal{B}_L}{\sup} \left\{\mathbb{E}_\mu[\phi] - \overline{\phi}_{\mathcal{B}_\text{med}^{\mathbf{X}}}  +\overline{\phi}_{\mathcal{B}_\text{med}^{\mathbf{Y}}} - \mathbb{E}_\nu[\phi] + \overline{\phi}_{\mathcal{B}_\text{med}^{\mathbf{X}}} -\overline{\phi}_{\mathcal{B}_\text{med}^{\mathbf{Y}}} \right\},\nonumber\\
&\leq \underset{\phi \in \mathcal{B}_L}{\sup} ~~ \underset{\substack{1 \le k \le K_{\mathbf{X}}\\[0.1cm]1 \le l \le K_{\mathbf{Y}}}}{\text{med}} \Big\{ \mathbb{E}_\mu[\phi] - \overline{\phi}_{\mathbf{X}, k}  +\overline{\phi}_{\mathbf{Y}, l} - \mathbb{E}_\nu[\phi] \Big\} + \Wmoucb(\hat{\mu}_n, \hat{\nu}_m),
\end{align}
where $\mathcal{B}^{\mathbf{X}}_\text{med}$ and $\mathcal{B}^{\mathbf{Y}}_\text{med}$ are the median blocks of $\overline{\phi}_{\mathbf{X},k} - \overline{\phi}_{\mathbf{Y},l}$ for $1 \le k \le K_{\mathbf{X}}$ and $1 \le l \le K_{\mathbf{Y}}$.
From \Cref{eq:major_1,eq:major_2}, we deduce that:
\begin{align}\label{eq:bounds}
\big \vert \Wmoucb(\hat{\mu}_n, \hat{\nu}_m) - \mathcal{W}(\mu, \nu) \big \vert &\leq \underset{\phi \in \mathcal{B}_L}{\sup } ~~ \underset{\substack{1 \le k \le K_{\mathbf{X}}\\[0.1cm]1 \le l \le K_{\mathbf{Y}}}}{\text{med}} \Big\{ \big \vert \overline{\phi}_{\mathbf{X},k}-\mathbb{E}_\mu[\phi] + \mathbb{E}_\nu[\phi] - \overline{\phi}_{\mathbf{Y},l} \big \vert \Big\},\\[0.2cm]
&\le  \underset{k \in \mathcal{I}_{\mathbf{X}},~l \in \mathcal{I}_{\mathbf{Y}}}{\text{sup}} ~~ \underset{\phi \in \mathcal{B}_L}{\sup} ~ \big \vert \overline{\phi}_{\mathbf{X},k}-\mathbb{E}_\mu[\phi] + \mathbb{E}_\nu[\phi] - \overline{\phi}_{\mathbf{Y},l} \big \vert,\nonumber\\[0.2cm]
&\le  \underset{k \in \mathcal{I}_{\mathbf{X}}}{\text{sup}} ~ \underset{\phi \in \mathcal{B}_L}{\sup } \big \vert \overline{\phi}_{\mathbf{X},k}-\mathbb{E}_\mu[\phi]\big \vert + \underset{l \in \mathcal{I}_{\mathbf{Y}}}{\text{sup}} ~ \underset{\phi \in \mathcal{B}_L}{\sup } \big \vert \mathbb{E}_\nu[\phi] - \overline{\phi}_{\mathbf{Y},l} \big \vert,\nonumber
\end{align}

where we have used  the fact that $\mathcal{I}_{\mathbf{X}} \times \mathcal{I}_{\mathbf{Y}}$ represents a majority of blocks, and the subadditivity of the supremum.
By independence between samples \textbf{X} and \textbf{Y}, and between the blocks, it holds:
\begin{align*}
&\mathbb{P} \left\{ \big \vert  \mathcal{W}_\mathrm{MoU}(\hat{\mu}_n, \hat{\nu}_m) - \mathcal{W}(\mu, \nu) \big \vert \underset{\substack{n\rightarrow +\infty \\ m \rightarrow + \infty}}{\longrightarrow} 0  \right\}\\
\geq~&\prod_{k \in \mathcal{I}_{\mathbf{X}}} \mathbb{P}\left\{ \underset{\phi \in \mathcal{B}_L}{\sup} \big\vert \overline{\phi}_{\mathbf{X},k}-\mathbb{E}_\mu[\phi] \big\vert \underset{\substack{n\rightarrow +\infty }}{\longrightarrow} 0  \right\} \cdot \prod_{l \in \mathcal{I}_{\mathbf{Y}}} \mathbb{P}\left\{ \underset{\phi \in \mathcal{B}_L}{\sup} \big\vert \overline{\phi}_{\mathbf{Y},l} -\mathbb{E}[\phi] \big\vert \underset{\substack{ m \rightarrow + \infty}}{\longrightarrow} 0  \right\}.
\end{align*}

Now, the arguments to get the right-hand side equal to $1$ are similar to those used in Lemma 3.1 and Proposition 3.2 in \cite{Bharath2012}.
We expose them explicitly for the sake of clarity.

Let $\mathcal{N}(\varepsilon, \mathcal{B}_L, L^1(\mu))$ be the \emph{covering number} of $\mathcal{B}_L$ which is the minimal number of $L^1(\mathbb{\mu})$ balls of radius $\varepsilon$ needed to cover $\mathcal{B}_L$.
Let $\mathcal{H}(\varepsilon, \mathcal{B}_L, L^1(\mu))$ be the \emph{entropy} of $\mathcal{B}_L$, defined as $\mathcal{H}(\varepsilon, \mathcal{B}_L, L^1(\mu))=\log \mathcal{N}(\varepsilon, \mathcal{B}_L, L^1(\mu)) $.
Let $F$ be the minimal enveloppe function such that $F(x)= \sup_{\phi \in\mathcal{B}_L} \vert \phi(x) \vert$.
We need to check that $\int F d\mu$ and $\int F d\nu$ are finite and that  $(1/n) \mathcal{H}(\varepsilon,\mathcal{B}_L, L^1(\hat{\mu}_n))$ and $(1/m)\mathcal{H}(\varepsilon,\mathcal{B}_L, L^1(\hat{\nu}_m))$ go to zero when $n$ and $m$ go to infinity.
Then, we can apply Theorem 3.7 in \cite{vdG00} which ensures the uniform (a.s.) convergence of empirical processes.
For any $\phi \in \mathcal{B}_L$, one has
\begin{equation}\label{support}
\phi(x)\leq \underset{x\in \mathcal{K}}{ \sup} |\phi(x)| \leq\underset{x,y\in \mathcal{K}}{ \sup} |\phi(x)-\phi(y)| \leq \underset{x,y\in \mathcal{K}}{\sup} \|x-y\|= \text{diam}(\mathcal{K}) < +\infty.
\end{equation}

Therefore $F(x)$ is finite, and following Lemma 3.1. in \cite{kolmogorovtihomirov} we have
\begin{equation*}
\mathcal{H}(\varepsilon,\mathcal{B}_L, \|\cdot\|_{\infty} )\leq \mathcal{N}(\varepsilon/4, \mathcal{K}, \| \cdot \|_2 ) \log \left(2 \left\lceil \frac{2 \text{diam} (\mathcal{K})}{\varepsilon}\right\rceil +1\right).
\end{equation*}

Since $\mathcal{H}(\varepsilon,\mathcal{B}_L, L^1(\hat{\mu}_n)) \leq \mathcal{H}(\varepsilon,\mathcal{B}_L, \|\cdot\|_\infty)$ and $\mathcal{H}(\varepsilon,\mathcal{B}_L, L^1(\hat{\nu}_m)) \leq \mathcal{H}(\varepsilon,\mathcal{B}_L, \|\cdot\|_\infty)$ then when, respectively, $n$ and $m$ go to infinity, we have
\begin{equation*}
\frac{1}{n} \mathcal{H}(\varepsilon,\mathcal{B}_L, L^1(\hat{\mu}_n)) \overset{\mu}{\longrightarrow} 0, \qquad \text{and} \qquad \frac{1}{m} \mathcal{H}(\varepsilon,\mathcal{B}_L, L^1(\hat{\nu}_m)) \overset{\nu}{\longrightarrow} 0,
\end{equation*}
which leads to the desired result.

{\bf Adaptation to other estimators.}
The above proof can be adapted in a straightforward fashion to $\mathcal{W}(\hat{\mu}_\mathrm{MoM}, \mu)$ and $\mathcal{W}_\mathrm{MoU-diag}(\hat{\mu}_n, \hat{\nu}_m)$.
Indeed, it holds
\begin{equation*}
\mathcal{W}(\hat{\mu}_\text{MoM}, \mu) = \sup_{\phi \in \mathcal{B}_L} ~ \underset{1 \le k \le K_{\mathbf{X}}}{\text{med}} ~ \left| \overline{\phi}_{\mathbf{X}, k} - \mathbb{E}_\mu\left[\phi\right] \right|,
\end{equation*}
and
\begin{equation*}
\Big| \mathcal{W}_\mathrm{MoU-diag}(\hat{\mu}_n, \hat{\nu}_m) - \mathcal{W}(\mu, \nu) \Big| \le \sup_{\phi \in \mathcal{B}_L} ~ \underset{\substack{1 \le k \le K_{\mathbf{X}}\\[0.1cm]1 \le l \le K_{\mathbf{Y}}}}{\text{med}} ~ \left| \overline{\phi}_{\mathbf{X}, k} - \mathbb{E}_\mu\left[\phi\right] + \mathbb{E}_\nu\left[\phi\right] - \overline{\phi}_{\mathbf{Y}, k} \right|.
\end{equation*}

It is then direct to adapt the reasoning from \Cref{eq:bounds}.
\qed

\subsection{Proof of Proposition 5}

Let $\psi \in \mathcal{B}_L$.
From \Cref{support}, we know that $-\text{diam}(\mathcal{K}) \le \psi(X) \le \text{diam}(\mathcal{K})$, so that $\psi(X)$ is in particular sub-Gaussian with parameter $\lambda = \text{diam}(\mathcal{K})$.
A direct application of Proposition 1 in \cite{papier2} then gives that for all $\delta \in ]0, e^{-4n\sqrt{2\tau_{\mathbf{X}}}}]$ and $K_{\mathbf{X}}= \lceil\sqrt{2\tau_{\mathbf{X}}} n \rceil$ , it holds with probability at least $1 - \delta$:
\begin{equation}\label{eq:dev_psi}
\Big| \mathrm{MoM}_{\mathbf{X}}[\psi] - \mathbb{E}_\mu\left[\psi\right]\Big| \le 4~\text{diam}(\mathcal{K})~\Gamma(\tau_{\mathbf{X}}) ~ \sqrt{\frac{\log(1/\delta)}{n}},
\end{equation}
with $\Gamma\colon \tauX \mapsto \sqrt{1 + \sqrt{2\tauX}} / \sqrt{1 - 2\tauX}$.
Using \Cref{lem:median}, observe also that $\forall (\phi, \psi)\in \mathcal{B}_L^2$ it holds:
\begin{align}\label{eq:major}
\big|\mathrm{MoM}_{\mathbf{X}}[\phi] - \mathbb{E}_\mu\left[\phi\right]\big| &\le \big|\mathrm{MoM}_{\mathbf{X}}[\phi] - \mathrm{MoM}_{\mathbf{X}}[\psi]\big| + \big|\mathbb{E}_\mu\left[\phi\right] - \mathbb{E}_\mu\left[\psi\right]\big| + \big|\mathrm{MoM}_{\mathbf{X}}[\psi] - \mathbb{E}_\mu\left[\psi\right]\big|,\nonumber\\
&\le 2 \|\phi - \psi\|_\infty + \big|\mathrm{MoM}_{\mathbf{X}}[\psi] - \mathbb{E}_\mu\left[\psi\right]\big|.
\end{align}

Now, let $\zeta > 0$, and $\psi_1, \ldots, \psi_{\mathcal{N}(\zeta, \mathcal{B}_L, \|\cdot\|_\infty)}$ be a $\zeta$-coverage of $\mathcal{B}_L$ with respect to $\|\cdot\|_\infty$.
We know from \cite{Bharath2012} that there exists $C_L > 0$ such that for all $\zeta > 0$ it holds:
\begin{equation}\label{eq:entropy}
\log(\mathcal{N}(\zeta, \mathcal{B}_L, \|\cdot\|_\infty)) \le C_L^2 (1/\zeta)^d
\end{equation}

From now on, we use $\mathcal{N}=\mathcal{N}(\zeta, \mathcal{B}_L, \|\cdot\|_\infty)$ for notation simplicity.
Let $\phi$ be an arbitrary element of $\mathcal{B}_L$.
By definition, there exists $i \le \mathcal{N}$ such that $\|\phi - \psi_i\|_\infty \le \zeta$.
\Cref{eq:major} then gives:
\begin{equation}\label{eq:decompo}
\Big|\mathrm{MoM}_{\mathbf{X}}[\phi] - \mathbb{E}_\mu\left[\phi\right]\Big| \le 2\zeta + \Big|\mathrm{MoM}_{\mathbf{X}}[\psi_i] - \mathbb{E}_\mu\left[\psi_i\right]\Big|.
\end{equation}

Applying \Cref{eq:dev_psi} to every $\psi_i$, the union bound gives that with probability at least $1 - \delta$ it holds:
\begin{equation*}
\sup_{i \le \mathcal{N}} \Big|\mathrm{MoM}_{\mathbf{X}}[\psi_i] - \mathbb{E}_\mu\left[\psi_i\right]\Big| \le 4 ~ \text{diam}(\mathcal{K})~\Gamma(\tau_{\mathbf{X}})~ \sqrt{\frac{\log(\mathcal{N}/\delta)}{n}}.
\end{equation*}

Taking the supremum in both sides of \Cref{eq:decompo}, it holds with probability at least $1 - \delta$:
\begin{equation*}
\sup_{\phi \in \mathcal{B}_L} \Big|\mathrm{MoM}_{\mathbf{X}}[\phi] - \mathbb{E}_\mu\left[\phi\right]\Big| \le 2\zeta + 4~\text{diam}(\mathcal{K})~\Gamma(\tau_{\mathbf{X}}) ~ \sqrt{\frac{C_L^2 \zeta^{-d} + \log(1/\delta)}{n}}.
\end{equation*}

Choosing $\zeta \sim 1/ n^{1/(d+2)}$ and breaking the square root finally gives that it holds with probability at least $1 - \delta$:
\begin{equation*}
\sup_{\phi \in \mathcal{B}_L} \Big|\mathrm{MoM}_{\mathbf{X}}[\phi] - \mathbb{E}_\mu\left[\phi\right]\Big| \le \frac{C_1(\tau_{\mathbf{X}})}{n^{1/(d+2)}} + C_2(\tau_{\mathbf{X}})\sqrt{\frac{\log(1/\delta)}{n}},
\end{equation*}
with $C_1(\tau_{\mathbf{X}}) = 2 + C_L C_2(\tau_{\mathbf{X}})$, and $C_2(\tau_{\mathbf{X}}) = 4~\text{diam}(\mathcal{K})~\Gamma(\tau_{\mathbf{X}})$.
\bigskip

{\bf Adaptation to MoU.}
From \Cref{support}, we get that the kernel $h_\phi\colon (X, Y) \mapsto \phi(X) - \phi(Y)$ has finite essential supremum $\|h_\phi(X, Y)\|_\infty \le \text{diam}(\mathcal{K})$.
Using Proposition 4 in \cite{papier2} with the same reasoning as above leads to the desired result, multiplying constants by a $2$ factor.
\qed

\subsection{Proof of Theorem 7}

Since $n^{\frac{1}{d+2} + \frac{1 - \beta}{2}} \ge C_1(\tauX)/(2C_2(\tauX)(2\tauX)^\frac{1}{4})$, then for all $\delta \in ]0, e^{-4n\sqrt{2\tau_{\mathbf{X}}}}]$, it holds:
\begin{equation*}
\frac{C_1(\tau_{\mathbf{X}})}{n^{1/(d+2)}} \le C_2(\tau_{\mathbf{X}}) \sqrt{\frac{4n \sqrt{2\tau_{\mathbf{X}}}}{n^\beta}} \le C_2(\tau_{\mathbf{X}}) \sqrt{\frac{\log(1/\delta)}{n^\beta}}.
\end{equation*}

One then has:
\begin{equation*}
\mathcal{W}(\hat{\mu}_\mathrm{MoM}, \mu) \ge 0 \ge \frac{C_1(\tau_{\mathbf{X}})}{n^{1/(d+2)}} - C_2(\tau_{\mathbf{X}})\sqrt{\frac{\log(1/\delta)}{n^\beta}}.
\end{equation*}

Combining with the first results of Proposition 4, for all $\delta \in ]0, e^{-4n\sqrt{2\tau_{\mathbf{X}}}}]$, it holds with probability at least $1 - \delta$:
\begin{equation*}
\left| \mathcal{W}(\hat{\mu}_\mathrm{MoM}, \mu) - \frac{C_1(\tau_{\mathbf{X}})}{n^{1/(d+2)}}\right| \le C_2(\tau_{\mathbf{X}})\sqrt{\frac{\log(1/\delta)}{n^\beta}}.
\end{equation*}
Reverting the inequation gives that it holds
\begin{equation}\label{eq:dev}
\mathbb{P}\left\{\left| \mathcal{W}(\hat{\mu}_\mathrm{MoM}, \mu) - \frac{C_1(\tau_{\mathbf{X}})}{n^{1/(d+2)}}\right| > t \right\} \le e^{-n^\beta t^2/{C_2}^2(\tau_{\mathbf{X}})},
\end{equation}
for all $t$ such that
\begin{equation}\label{eq:t}
t \ge (32\;\tau_{\mathbf{X}})^{1/4}C_2(\tau_{\mathbf{X}})~\sqrt{n^{1 - \beta}} = \frac{(32\;\tau_{\mathbf{X}})^{1/4}}{\sqrt{\tau_{\mathbf{X}}}}C_2(\tau_{\mathbf{X}})\sqrt{n^{1 - \beta}~\frac{n_\mathcal{O}}{n}}.
\end{equation}

%
One may finally use that for a nonnegative random variable it holds:

\begin{align}
\mathbb{E}\left| \mathcal{W}(\hat{\mu}_\mathrm{MoM}, \mu) - \frac{C_1(\tau_{\mathbf{X}})}{n^{1/(d+2)}}\right| &= \int_0^\infty \mathbb{P}\left\{\left| \mathcal{W}(\hat{\mu}_\mathrm{MoM}, \mu) - \frac{C_1(\tau_{\mathbf{X}})}{n^{1/(d+2)}}\right| > t \right\}dt,\nonumber\\[0.3cm]
&\le \int_0^{\frac{(32\;\tau_{\mathbf{X}})^{1/4}}{\sqrt{\tau_{\mathbf{X}}}}C_{\mathcal{O}}C_2(\tau_{\mathbf{X}})\sqrt{n^{\alpha_{\mathcal{O}} - \beta}}}1dt + \int_0^\infty e^{-n^\beta t^2/{C_2}^2(\tau_{\mathbf{X}})}dt,\nonumber\\[0.3cm]
&\le \frac{(32\;\tau_{\mathbf{X}})^{1/4}}{\sqrt{\tau_{\mathbf{X}}}}~\frac{C_\mathcal{O}C_2(\tau_{\mathbf{X}})}{n^{(\beta- \alpha_\mathcal{O})/2}} + \frac{\sqrt{\pi}~C_2(\tau_{\mathbf{X}})}{2~n^{\beta/2}}.\nonumber\\[0.3cm]
&= 2\;(2/\tauX)^{1/4}~\frac{C_\mathcal{O}C_2(\tau_{\mathbf{X}})}{n^{(\beta- \alpha_\mathcal{O})/2}} + \frac{\sqrt{\pi}~C_2(\tau_{\mathbf{X}})}{2~n^{\beta/2}}.\label{eq:final}
\end{align}

Where the second line holds thanks to Assumption 6.
\bigskip

{\bf Adaptation to MoU.}
The adaptation is straightforward, up to \Cref{eq:t}, that now writes:
\begin{align*}
t &\ge 2 \times (32(\tau_{\mathbf{X}} + \tau_{\mathbf{Y}}))^{1/4} C_2(\tau_{\mathbf{X}}  + \tau_{\mathbf{Y}} )~\sqrt{n^{1 - \beta}},\\[0.2cm]
&= 2 \times\frac{(32(\tau_{\mathbf{X}} + \tau_{\mathbf{Y}}))^{1/4}}{\sqrt{\tau_{\mathbf{X}} + \tau_{\mathbf{Y}} }}C_2(\tau_{\mathbf{X}}  + \tau_{\mathbf{Y}} )\sqrt{n^{1 - \beta}~\left(\frac{n_\mathcal{O}}{n} + \frac{m_\mathcal{O}}{m}\right)}.
\end{align*}
Using Assumption 6 on both samples \textbf{X} and \textbf{Y}, it leads to the desired results.
%
%
%
%
%
\qed
\clearpage

\section{Additional material of the numerical part}

In this part, we introduce algorithms and additional experiments that could not be in the paper for lack of space.
\subsection{Additional algorithms}

Here, algorithms  to compute $\Wmou(\mu_n,\nu_n)$ and $\Wmoucb(\mu_n,\nu_n)$ are displayed.

\begin{algorithm}[H]
\caption{Computation of $\Wmou(\mu_n,\nu_n)$. }

\textit{Initialization:} $\eta$, the learning rate. $c$, the clipping parameter.  $w_0$ the initial weights.
      \begin{algorithmic}[1]
          \For{$t=0,\ldots,n_{\text{iter}}$}

         \vspace*{0.1cm}
         \State Sample $K=\KX \wedge \KY$ disjoint blocks $\mathcal{B}^{\mathbf{XY}}_{1,1},\mathcal{B}^{\mathbf{XY}}_{2,2}, \ldots,\mathcal{B}^{\mathbf{XY}}_{k,k}, \ldots  \mathcal{B}^{\mathbf{XY}}_{K,K}$ from a sampling scheme 
         \vspace*{0.1cm}
         \State Find the median block $\mathcal{B}_{med}^{\mathbf{XY}}$ 
         \vspace*{0.1cm}
         \State \begin{align*} G_{w} \longleftarrow  & \bigl\lfloor  K / n  \bigr\rfloor \underset{(i,j) \in \mathcal{B}_{med}^{\mathbf{XY}}}{\sum}\nabla _{w}   \left[ \phi_w(X_i)- \phi_w(Y_j) \right]
         				\end{align*}   
         \vspace*{0.1cm}    
         \State ~7.1~$w \leftarrow w + \eta \times \text{RMSProp}(w, G_w)$
         \State ~7.2~ $w \leftarrow \text{clip}(w, -c,c)$
\EndFor \\
 \textbf{Output}: $w,\; \tWmoudiag, \; \phi_w$.
      \end{algorithmic}
      \label{algo:MoU_DIAG}
\end{algorithm}
\begin{algorithm}[H]
\caption{Computation of  $\Wmoucb(\mu_n,\nu_n)$. }

\textit{Initialization:} $\eta$, the learning rate. $c$, the clipping parameter.  $w_0$ the initial weights.
      \begin{algorithmic}[1]
          \For{$t=0,\ldots,n_{\text{iter}}$}

         \vspace*{0.1cm}
         \State Sample $\KX \times \KY$ disjoint blocks $\mathcal{B}^{\mathbf{XY}}_{1,1},\ldots,\mathcal{B}^{\mathbf{XY}}_{k,l}, \ldots  \mathcal{B}^{\mathbf{XY}}_{\KX,\KY}$ from a sampling scheme 
         \vspace*{0.1cm}
         \State Find the median block $\mathcal{B}_{med}^{\mathbf{XY}}$ 
         \vspace*{0.1cm}
         \State \begin{align*} G_{w} \longleftarrow  & \bigl\lfloor \KX /n \bigr\rfloor  \times \bigl\lfloor \KY / m \bigr\rfloor  \underset{(i,j) \in \mathcal{B}_{med}^{\mathbf{XY}}}{\sum}\nabla _{w}   \left[ \phi_w(X_i)- \phi_w(Y_j) \right]
         				\end{align*}   
         \vspace*{0.1cm}    
         \State ~7.1~$w \leftarrow w + \eta \times \text{RMSProp}(w, G_w)$
         \State ~7.2~ $w \leftarrow \text{clip}(w, -c,c)$
\EndFor \\
 \textbf{Output}: $w,\; \tWmoucb, \; \phi_w$.
      \end{algorithmic}
      \label{MWGAN}
\end{algorithm}
\subsection{Additional experiments}

In this part, numerical results for $\tWmoucb$ and $\tWmom$, related to the Section 4.2 of the paper, are displayed. Results of both experiments, depicted in Figure \ref{fig:error2} and \ref{fig:CV2}, are quite similar due to the simplicity of the problem. 

\clearpage
 \begin{figure}[h]
\begin{center}
\includegraphics[height=.6\textheight, trim=0cm .0cm 0cm 0cm,clip=true]{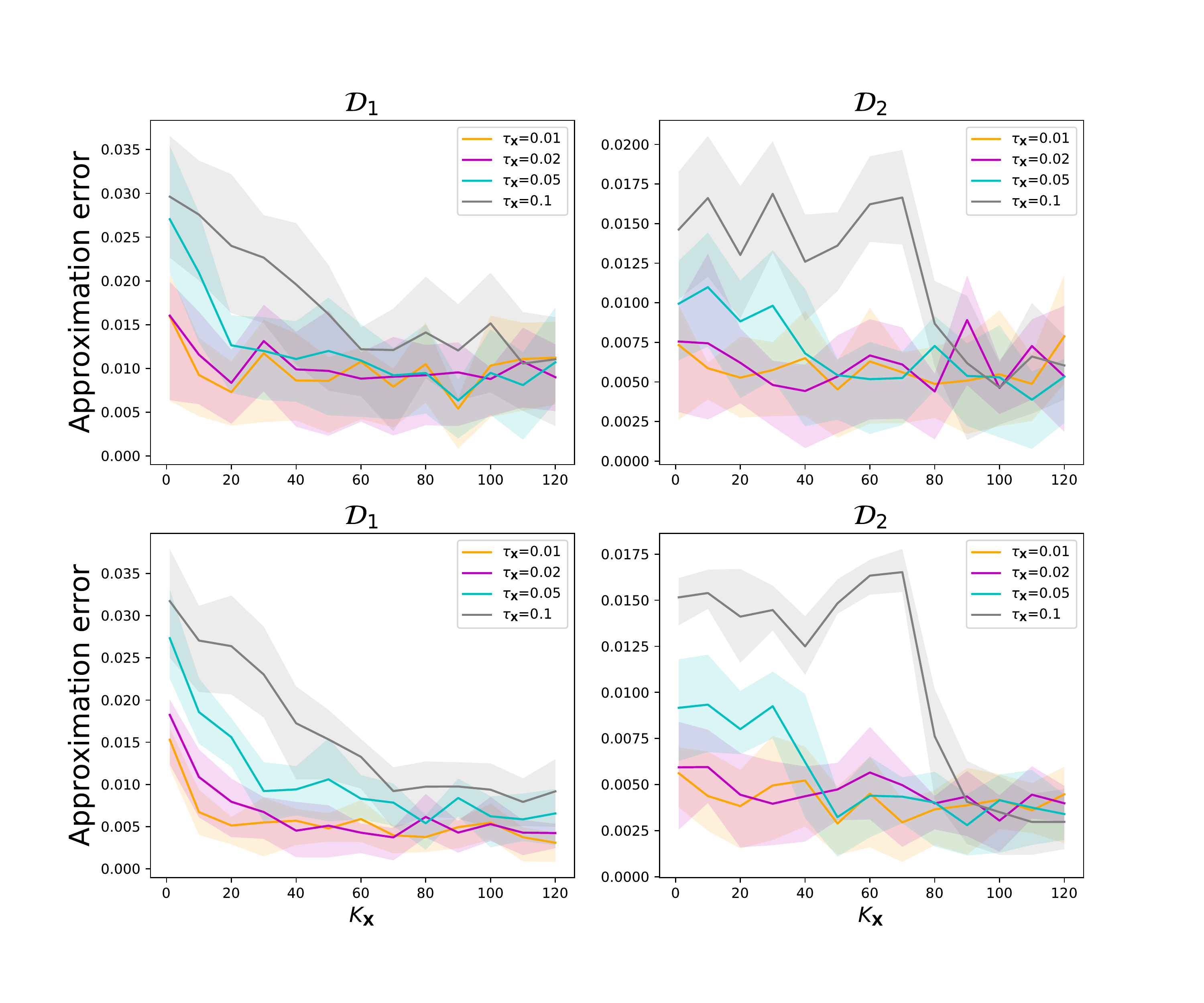}
\end{center}
\caption{$\tWmoucb$ (top) and $\tWmom$ (bottom)  over $\KX$ for different anomalies proportion $\tau_{X}$ on $\mathcal{D}_1$ (left) and $\mathcal{D}_2$ (right).}
\label{fig:error2}
\end{figure}
\clearpage
 \begin{figure}[h]
\begin{center}
\includegraphics[height=.5\textheight, trim=0cm .0cm 0cm 0cm,clip=true]{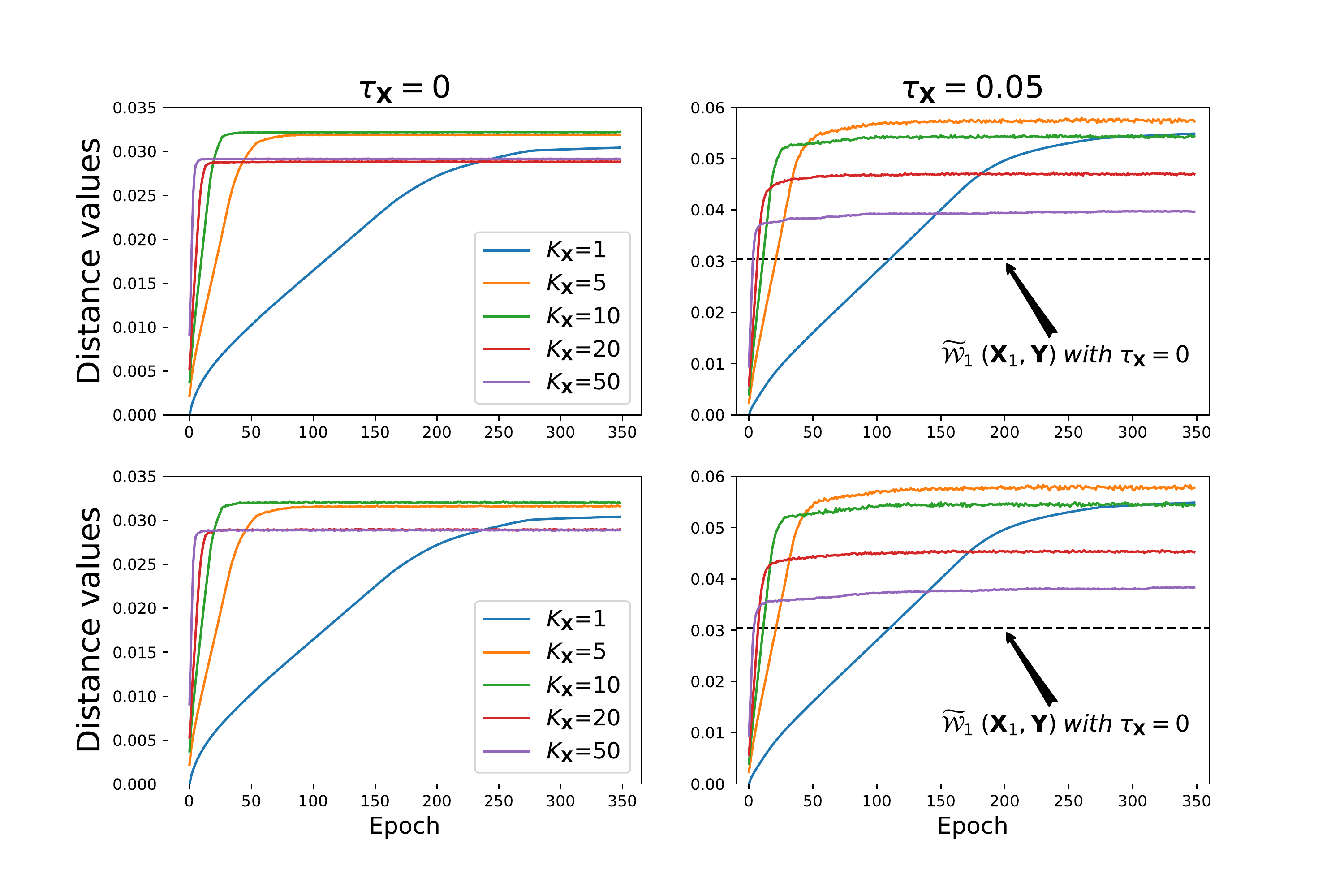}
\caption{Convergence of $\tWmoucb$ (top) and $\tWmom$ (bottom) without anomalies (left) and with 5\% anomalies (right) for different $\KX$.}
\label{fig:CV2}
\end{center}
\end{figure}
\end{document}